\documentclass{article}

\usepackage{amsmath,amsfonts,bm}

\def\eqref#1{equation~\ref{#1}}

\def\1{\bm{1}}

\DeclareMathAlphabet{\mathsfit}{\encodingdefault}{\sfdefault}{m}{sl}
\SetMathAlphabet{\mathsfit}{bold}{\encodingdefault}{\sfdefault}{bx}{n}

\def\gN{{\mathcal{N}}}

\newcommand{\E}{\mathbb{E}}

\DeclareMathOperator*{\argmin}{arg\,min}

\usepackage{microtype}
\usepackage{graphicx}
\usepackage{subfig}
\usepackage{booktabs} 
\usepackage{url}
\usepackage{graphicx}
\usepackage{xcolor}
\usepackage{setspace}
\usepackage{wrapfig}
\usepackage{algorithm}
\usepackage{algorithmic}
\usepackage{comment}
\usepackage{enumitem}
\usepackage{amsmath}
\usepackage{float}
\usepackage{balance} 
\usepackage{mathtools}

\newtheorem{theorem}{Theorem}
\newtheorem{lemma}{Lemma}
\newtheorem{proof}{Proof}
\newtheorem{corollary}{Corollary}
\usepackage{thm-restate}

\usepackage{hyperref}

\usepackage[accepted]{icml2021}

\begin{document}

\twocolumn[
\icmltitle{To be Robust or to be Fair: Towards Fairness in Adversarial Training}



\icmlsetsymbol{equal}{*}

\begin{icmlauthorlist}
\icmlauthor{Han Xu}{equal,to}
\icmlauthor{Xiaorui Liu}{equal,to}
\icmlauthor{Yaxin Li}{to}
\icmlauthor{Anil K. Jain}{to}
\icmlauthor{Jiliang Tang}{to}
\end{icmlauthorlist}

\icmlaffiliation{to}{Department of Computer Science and Engineering, Michigan State University, Michigan, U.S.}

\icmlcorrespondingauthor{Han Xu}{xuhan1@msu.edu}
\icmlcorrespondingauthor{Xiaorui Liu}{xiaorui@msu.edu}

\icmlkeywords{Machine Learning, ICML}

\vskip 0.3in
]



\printAffiliationsAndNotice{\icmlEqualContribution} 

\begin{abstract}
Adversarial training algorithms have been proved to be reliable to improve machine learning models' robustness against adversarial examples. However, we find that adversarial training algorithms tend to introduce severe disparity of accuracy and robustness between different groups of data. For instance, a PGD adversarially trained ResNet18 model on CIFAR-10 has 93\% clean accuracy and 67\% PGD $l_{\infty}$-8 robust accuracy on the class ''automobile'' but only 65\% and 17\% on the class ''cat''. This phenomenon happens in balanced datasets and does not exist in naturally trained models when only using clean samples. In this work, we empirically and theoretically show that this phenomenon can happen under general adversarial training algorithms which minimize DNN models' robust errors. Motivated by these findings, we propose a Fair-Robust-Learning (FRL) framework to mitigate this unfairness problem when doing adversarial defenses. Experimental results validate the effectiveness of FRL.
\end{abstract}

\section{Introduction}

\begin{figure*}[h]
\subfloat[Natural Training]{\label{clean}
\begin{minipage}[c]{0.24\textwidth}
\centering
\includegraphics[width = 1\textwidth]{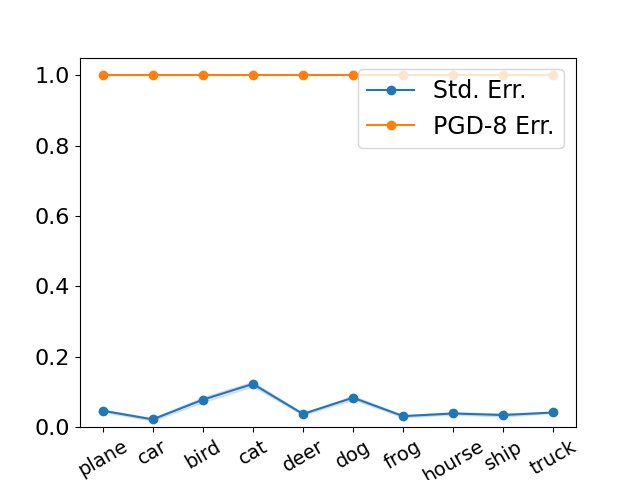}
\end{minipage}
}
\subfloat[PGD Adversarial Training]{
\begin{minipage}[c]{0.24\textwidth}
\centering
\includegraphics[width = 1\textwidth]{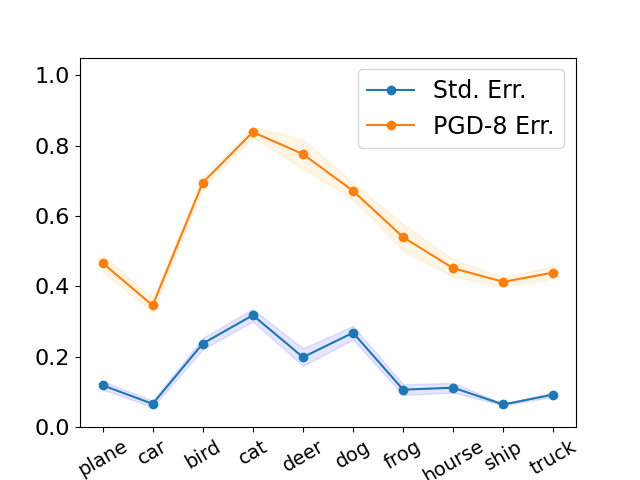}
\end{minipage}
}
\subfloat[TRADES ($1/\lambda = 1$)]{ 
\begin{minipage}[c]{0.24\textwidth}
\centering
\includegraphics[width = 1\textwidth]{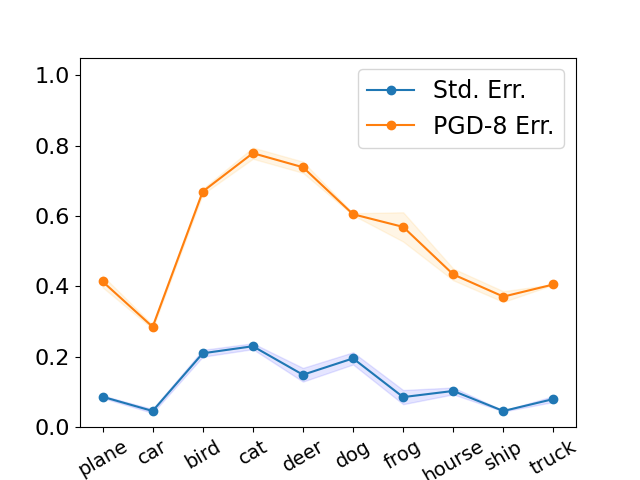}
\end{minipage}
}
\subfloat[TRADES ($1/\lambda = 6$)]{
\begin{minipage}[c]{0.24\textwidth}
\centering
\includegraphics[width = 1\textwidth]{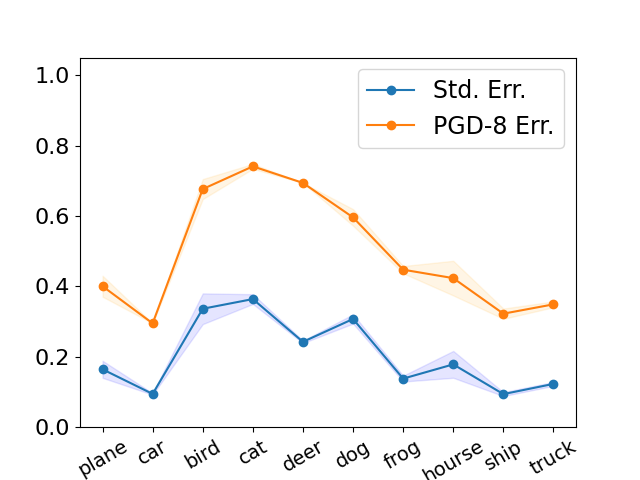}
\end{minipage}
}
\caption{The class-wise standard / robust error of natural and adversarial training methods on CIFAR10, under PreAct ResNet18. }
\label{fig:fair}
\end{figure*}

The existence of adversarial examples~\cite{goodfellow2014explaining, szegedy2013intriguing} causes great concerns when applying deep neural networks to safety-critical tasks, such as autonomous driving vehicles and face identification~\cite{morgulis2019fooling, sharif2016accessorize}. As countermeasures against adversarial examples, adversarial training algorithms aim to train a classifier that can classify the input samples correctly even when they are adversarially perturbed.  Namely, they optimize the model to have the minimum adversarial risk of a sample that can be perturbed to be wrongly classified: 
\begin{align}
    \min_f \underset{x}{\E}~  \left[\max_{||\delta||\leq\epsilon} \mathcal{L}(f(x+\delta),y)\right].`
    \label{eq:adv_training}
\end{align}
These adversarial training methods~\cite{kurakin2016adversarial, madry2017towards, zhang2019theoretically} have been shown to be one of the most effective and reliable approaches to improve the model robustness against adversarial attacks. 

Although promising to improve the model's robustness, we reveal an intriguing property about adversarial training algorithms: they usually result in a large disparity of accuracy and robustness among different classes. As a preliminary study in Section~\ref{pre}, we apply natural training and PGD adversarial training~\cite{madry2017towards} on the CIFAR10 dataset~\cite{krizhevsky2009learning} using a PreAct-ResNet18~\cite{he2016deep} architecture. For a naturally trained model, the model performance in each class is similar. However, in the adversarially trained model, there is a severe class-wise performance discrepancy (both accuracy and robustness). For example, the model has both low standard and robust errors on the samples from the class ``car'', but it has a much larger error rate on those ``cat'' images. Meanwhile, this fairness issue does not appear in natural models which are trained on clean data. Note that this phenomenon happens in a balanced CIFAR10 dataset and exists in other datasets, model structures and adversarial training algorithms. More details can be found in Section~\ref{pre}. If this phenomenon happens in real-world applications, it can raise huge concerns about safety. Imagine that a traffic sign recognizer has an overall high performance, but it is very inaccurate and vulnerable to perturbations only for some specific signs. The safety of this autonomous driving car is still not guaranteed. Meanwhile, this phenomenon can also lead to issues from the social ethics perspective. For example, a robustly trained face identification system might provide different levels of safety for the services provided to different ethnic communities. Thus, there is a pressing need to study this phenomenon.

In this work, we define this phenomenon as the ``\textit{robust fairness}'' problem of adversarial training and we aim to understand and mitigate this fairness problem. We first explore the question: ``\textit{why does adversarial training have this accuracy/robustness disparity between classes while natural training does not present a similar issue?}'' To answer this question, we first study on several cases on rather simple classification problems with mixture Gaussian distributions, where we design the data in different classes with different ``difficulty levels'' to be classified. By studying the class-wise performance of naturally and adversarially trained models, we deepen our understandings on this problem. Compared to natural training, adversarial training has a stronger tendency to favor the accuracy of the classes which are ``easier'' to be predicted.  Meanwhile, adversarial training will sacrifice the prediction performance for the ``difficult'' classes. As a result, there is a much obvious class-wise performance discrepancy in adversarial training. Motivated by our empirical findings and theoretical understandings, we propose a dynamic debiasing framework, i.e., Fair Robust Learning (FRL), to mitigate the robust fairness issue in adversarial settings. Our main contributions are summarized as: 
\begin{itemize}
  \setlength\itemsep{0em}
    \item We discover the robust fairness problem of adversarial training algorithms and empirically verify that this problem can be general;
    \item We build conceptual examples to understand the potential reasons that cause this fairness problem; and 
    \item We propose a Fair Robust Learning (FRL) framework to mitigate the fairness issue in adversarial training.
\end{itemize}

\section{Preliminary Studies}\label{pre}

In this section, we introduce our preliminary findings to show that adversarial training algorithms usually present the fairness issue, which is related to the strong disparity of standard accuracy and robustness among different classes. We first examine algorithms including PGD adversarial training~\citep{madry2017towards} and TRADES~\citep{zhang2019theoretically} on the CIFAR10 dataset~\citep{krizhevsky2009learning} under the $l_\infty
$-norm adversarial attack (under perturbation bound 8/255). In CIFAR10, we both naturally and adversarially train PreAct-ResNet18~\citep{he2009learning} models. The results are presented in Figure~\ref{fig:fair}. 


From Figure~\ref{fig:fair}, we can observe that -- for the naturally trained models, every class has similar standard error (around $7\pm5\%$) and $100\%$ robust error under the 8/255 PGD attack.  However, for adversarially trained models, the disparity phenomenon is severe. For example, a PGD-adversarially trained model has $32.8\%$ standard error and $82.4\%$ robust error for the samples in the class ``cat'', which are much larger than the model's average standard error $15.5\%$ and average robust error $56.4\%$. While, the best class ``car'' only has $6.1\%$ standard error and $34.3\%$ robust error. These results suggest that adversarial training can cause strong disparities of standard / robustness performance between different classes, which are originally negligible in natural training.

Moreover, from Figure~\ref{fig:fair}, we find that the reason of this fairness phenomenon might be due to the unequal influence of adversarial training on different classes. It tends to hurt the standard performance of classes which are intrinsically ``harder'' to be classified, but not effectively improve their robustness performance. In Table~\ref{tab:diff}, we list the classes ``dog'' and ``cat'', which have the highest errors in natural training, as well as ``car'' and ``ship'', which have the lowest errors. We can observe that adversarial training increases the standard errors of ``dog'' and ``cat'' by a much larger margin than the classes ``car'' and ``ship''. Similarly, adversarial training gives poorer help to reduce the robust errors of ``dog'' and ``cat''. As a conclusion, we hypothesize adversarial training tends to make the hard classes even harder to be classified or robustly classified. In Section~\ref{sec:toy}, we will theoretically confirm this hypothesis.

We further investigate other settings including model architecture WRN28, SVHN dataset, $l_2$-norm adversarial training and Randomized Smoothing~\cite{cohen2019certified} algorithm. Results can be found at Appendix~\ref{app:pre} where we can make similar observations. These findings suggest that observations from Figure~\ref{fig:fair} and Table~\ref{tab:diff} are likely to be generalized into other adversarial training algorithms, model architectures, datasets and adversarial attacks. 

\begin{table}[h]
\centering
\caption{The Changes of Standard \& Robust Error in Natural \& Adversarial Training in CIFAR10.}
\scalebox{0.9}{
\begin{tabular}{c| c c c c} 
\hline
Std. Error& \textbf{Cat} & \textbf{Dog}& \textbf{Car}& \textbf{Ship}\\
\hline
\textbf{Nat. Train} & 11.3 &10.0 &1.8 & 3.5 \\
\textbf{PGD Adv. Train} &34.8 & 26.9 & 6.1 &6.4\\
\hline
\textbf{Diff. (Adv. - Nat.)} &23.5 &16.9 &4.3 & 2.9\\
\hline
\end{tabular}}

\vspace{0.2cm}
\scalebox{0.9}{
\begin{tabular}{c| c c c c} 
\hline
Rob. Error & \textbf{Cat} & \textbf{Dog}& \textbf{Car}& \textbf{Ship}\\
\hline
\textbf{Nat. Train} & 100 &100 &100 &100 \\
\textbf{PGD Adv. Train} & 82.7 & 66.4 &34.3 &40.8\\
\hline
\textbf{Diff. (Adv. - Nat.)} &-17.3 &-33.5 &-65.7 &-59.2\\
\hline
\end{tabular}}

\label{tab:diff}
\end{table}

\section{Theoretical Analysis}\label{sec:toy}
From our preliminary studies, we consistently observe that adversarially trained models have great performance disparity (both standard and robust errors) between different classes. Moreover, adversarial training tends to hurt the classes which are originally harder to be classified in natural training. What is the reason that leads to this phenomenon? Is this ``unfairness'' property an inherent property of adversarial training methods? These questions are not trivial to answer because it is closely related to another property of adversarial training: it usually degrades the model's average accuracy. Given that naturally trained models already present slight disparities between classes (see Figure~\ref{fig:fair} (a)), is the larger class-wise accuracy disparity in adversarial training only a natural outcome of the model's worse average accuracy? 

To deepen our understandings on these questions, we theoretically study the effect of adversarial training on a binary classification task under a mixture Gaussian distribution. We design the two classes with different ``difficulties'' to be classified. In such case, adversarial training will not significantly degrade the average standard error, but its decision boundary, compared to naturally trained models, are more biased to favor the ``easier'' class and hurt the ``harder'' class.
From this theoretical study, we aim to validate that adversarial training methods can have the intrinsic property to give unequal influence between different classes and consequently cause the fairness problem. In the following, we first introduce the necessary notations and definitions.

\noindent \textbf{Notation.} We use $f$ to denote the classification model which is a mapping $f: \mathcal{X}\rightarrow \mathcal{Y}$ from input data space $\mathcal{X}$ and output labels $\mathcal{Y}$. Generally, for a classifier $f$, the overall \textit{standard error} is defined as $\mathcal{R}_\text{nat}(f) = \text{Pr.}(f(x)\neq y)$; and its overall \textit{robust error} is $\mathcal{R}_\text{rob}(f) = \text{Pr.}(\exists\delta, ||\delta||\leq\epsilon, \text{s.t.} f(x+\delta)\neq y)$, which is the probability that there exists a perturbation to make the model give a wrong prediction.\footnote{In this section, we focus on $l_\infty$-norm bounded perturbation.} We use $\mathcal{R}_\text{nat}(f; y)$ to denote the standard error conditional on a specific class $Y = y$.


\subsection{A Binary Classification Task}\label{sec:binary_example}


We start by giving a rather simple example of a binary classification task with a Gaussian mixture data. Here, we aim to design the two classes with inherent different ``difficulties'' to be classified. 
Specifically, in the following definition, the data are from 2 classes $\mathcal{Y} = \{-1, +1\}$ and the data from each class follow a Gaussian distribution $\mathcal{D}$ which is centered on $-\theta$ and $\theta$ respectively. In our case, we specify that there is a $K$-factor difference between two classes' variance: $\sigma_{+1}:\sigma_{-1} = K:1$ and $K>1$. 
\begin{spreadlines}{0.8em}
\begin{align}\small
\label{eq:data_dist1}
\begin{split}
    & y \stackrel{u.a.r}{\sim} \{-1, +1\},~~~ \theta = ( \overbrace{\eta,...,\eta}^\text{dim = d}),~~~
    \\ &x \sim  
    \begin{cases}
      ~~\mathcal{N}~(\theta, \sigma_{+1}^2I) ~~~~&\text{if $y= + 1$}\\
      ~~\mathcal{N}~(-\theta, \sigma_{-1}^2I) ~~~~&\text{if $y= - 1$}\\
    \end{cases} 
\end{split}
\end{align}
\end{spreadlines}

Intuitively, the class ``+1'' is harder than class ``-1'' because it is less compacted in the data space. In Theorem~\ref{thm:error_natural_train}, we formally show that the class ``+1'' can be indeed harder because an optimal linear classifier will give a larger error for the class ``+1'' than class ``-1''.


\begin{theorem} 
\label{thm:error_natural_train} For a data distribution $\mathcal{D}$ in Eq.~\ref{eq:data_dist1}, the optimal linear classifier $f_\text{nat}$ which minimizes the average standard classification error:
\begin{align*}
    f_\text{nat} = \argmin_f \text{Pr.}(f(x)\neq y)
\end{align*}
It has the intra-class standard error for the two classes:
\begin{align}
\label{eq:nat_error}
\small
\begin{split}
&\mathcal{R}_\text{nat}(f_\text{nat},-1) 
= \text{Pr.} \{\mathcal{N}(0,1)\leq  A  - K\cdot\sqrt{A^2 + q(K)}  \}
\\
& \mathcal{R}_\text{nat}(f_\text{nat},+1) 
=  \text{Pr.} \{\mathcal{N}(0,1)\leq  -K\cdot A + \sqrt{A^2 + q(K)}\}
\end{split}
\end{align}
where $A = \frac{2}{K^2-1}\frac{\sqrt{d}\eta}{\sigma}$ and $q(K) = \frac{2\log K}{K^2-1}$ which is a positive constant and only depends on $K$,  As a result, the class ``+1'' has a larger standard error:
\begin{equation*}
    \mathcal{R}_\text{nat}(f_\text{nat},-1) <\mathcal{R}_\text{nat}(f_\text{nat},+1).
\end{equation*}
\end{theorem}


A detailed proof of Theorem~\ref{thm:error_natural_train} can be found in Appendix~\ref{app:thy}.  From Theorem~\ref{thm:error_natural_train}, it demonstrates that class ``+1'' (with large variance) can be harder to be classified than class ``-1'', because an optimal natural classifier will present a larger standard error in class ``+1''. Note that the classwise difference is due to the positive term $q(K)$ in Eq.~\ref{eq:nat_error}, which depends on the variance ratio $K$. If the two classes' variances are equal, i.e., $K = 1$ , the standard errors for the two classes are the same. Next, we will show that, in the setting of Eq.~\ref{eq:data_dist1}, an optimal robust classifier (adversarial training) will give a model which further benefits the easier class ``-1'' and hurt the harder class ``+1''.

\subsection{Optimal Linear Model to Minimize Robust Error}

In this subsection, we demonstrate that an adversararially trained model exacerbates the performance gap between these two classes, by giving a decision boundary which is closer to samples in the harder class ``+1''  and farther to the class ``-1''. Similar to Theorem~\ref{thm:error_natural_train}, we  calculate the classwise standard errors for robust classifiers.
\begin{theorem} 
\label{thm:error_adv_train} For a data distribution $\mathcal{D}$ in Eq.~\ref{eq:data_dist1}, the optimal robust linear classifier $f_\text{rob}$ which minimizes the average robust error:
\begin{align*}
    f_\text{rob} =  \argmin_f\text{Pr.}(\exists ||\delta||\leq \epsilon~~ \text{s.t.}~~ f(x+\delta)\neq y)
\end{align*}
It has the intra-class standard error for the two classes:
\begin{align}
\label{eq:nat_error}
\small
\begin{split}
&\mathcal{R}_\text{nat}(f_\text{rob},-1)\\ 
= &\text{Pr.} \{\mathcal{N}(0,1)\leq  B  - K\cdot\sqrt{B^2 + q(K)} -\frac{\sqrt{d}}{\sigma}\epsilon  \}
\\
& \mathcal{R}_\text{nat}(f_\text{rob},+1)\\ 
=  &\text{Pr.} \{\mathcal{N}(0,1)\leq  -K\cdot B + \sqrt{B^2 + q(K)}-\frac{\sqrt{d}}{K\sigma}\epsilon\}
\end{split}
\end{align}
where $B = \frac{2}{K^2-1}\frac{\sqrt{d}(\eta-\epsilon)}{\sigma}$ and  $q(K) = \frac{2\log K}{K^2-1}$ is a positive constant and only depends on $K$, 
\end{theorem}
Note that we limit the perturbation margin $\epsilon$ in the region $0<\epsilon<\eta$ to guarantee that the robust optimization gives a reasonable classification in this setting in Eq.~\ref{eq:data_dist1}. The detailed proof is similar to Theorem~\ref{thm:error_natural_train} and is given in Appendix~\ref{app:thy}. From the results in Theorems~\ref{thm:error_natural_train} and \ref{thm:error_adv_train}, an corollary can demonstrate that the robust classifier will further hurt the ``harder'' class's performance.

\begin{corollary}
\label{thm:enlarge_disparity}
Adversarially Trained Models on $\mathcal{D}$ will increase the standard error for class ``+1'' and reduce the standard error for class ``-1'':
\begin{align*}
\begin{split}
\mathcal{R}_\text{nat}(f_\text{rob},-1) <\mathcal{R}_\text{nat}(f_\text{nat},-1).  \\
\mathcal{R}_\text{nat}(f_\text{rob},+1) >\mathcal{R}_\text{nat}(f_\text{nat},+1).  
\end{split}
\end{align*}
\end{corollary}

\begin{proof}
\noindent A simplified proof sketch for Corollary~\ref{thm:enlarge_disparity} can help shed light on the reason of this behavior for adversarial training. First, from the definition of the distribution $\mathcal{D}$, it is easy to have the intermediate result that natural / robust classifiers have the weight vectors $w_\text{nat} = w_\text{rob}=\bf1$. The only difference is their interception terms $b_\text{nat}$ and $b_\text{rob}$. In Appendix~\ref{app:thy}, we show that:
\begin{align}
\label{eq:intercept}
\small
\begin{split}
b_\text{nat} = \frac{K^2+1}{K^2-1}d\eta - K\sqrt{\frac{4}{(K^2-1)^2}\cdot d^2\eta^2+d\sigma^2q(K)}:=g(\eta).
\end{split}
\end{align}
In particular, for robust classifiers, if we look at the objective of robust classification:
\begin{align*}\small
\begin{split}
\mathcal{R}_\text{rob}(f)=&\text{Pr.}(\exists ||\delta||\leq \epsilon~~ \text{s.t.}~~ f(x+\delta)\neq y)\\
=& \max_{||\delta||\leq\epsilon}{\text{Pr.}(f(x+\delta)\neq y)}\\
=& \frac{1}{2}\text{Pr.}(f(x+\epsilon)\neq -1|y = -1)
\\ +&\frac{1}{2} \text{Pr.}(f(x -\epsilon)\neq+1|y = +1)
\end{split}
\end{align*}
\noindent the robust classifier $f_\text{rob}$ directly minimizes the standard error of samples $(x-\bf\epsilon)$ for $x$ in class ``+1'', and it minimizes the error of samples $(x+\bf\epsilon)$ for $x$ in class ``-1''.
As a consequence, the robust model $f_\text{rob}$ minimizes the errors of samples whose centers are $\pm\theta' = \pm (\eta-\epsilon, ...,\eta-\epsilon)$, which are both $\epsilon$-distance closer to zero $\bf 0$ compared to $\pm \theta$. Thus, we can get the interception term of the robust model, by only replacing $\eta$ in $b_\text{nat}$ by ($\eta - \epsilon$): 
\begin{equation*}
b_\text{rob} = g(\eta - \epsilon).
\end{equation*}
In Appendix~\ref{app:thy}, we show $g$ is a monotone increasing function from 0 to $\eta$; thus we have the relation $0<b_\text{rob}<b_\text{nat}$. This suggests that a robust classifier will predict more samples in $\mathcal{D}$ to be a negative ``-1'' class, and hence reduce the classification error for class ``-1'' but increase the error of class ``+1''.
\end{proof}

\begin{figure}[t]
\subfloat[Logistic Regression]{\label{clean}
\begin{minipage}[c]{0.24\textwidth}
\centering
\includegraphics[width = 1\textwidth]{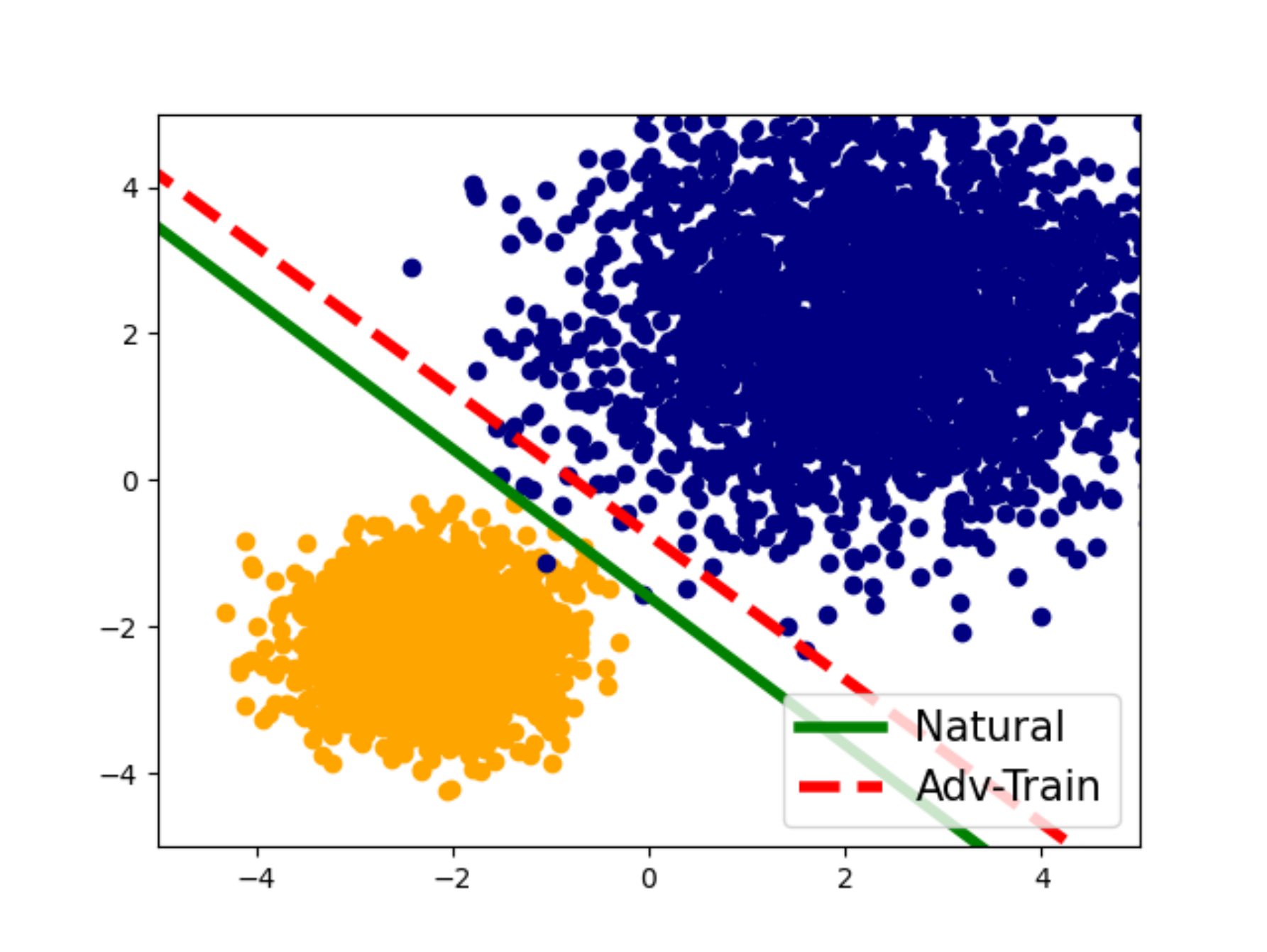}
\end{minipage}
}
\subfloat[One Layer Perceptron]{
\begin{minipage}[c]{0.24\textwidth}
\centering
\includegraphics[width = 1\textwidth]{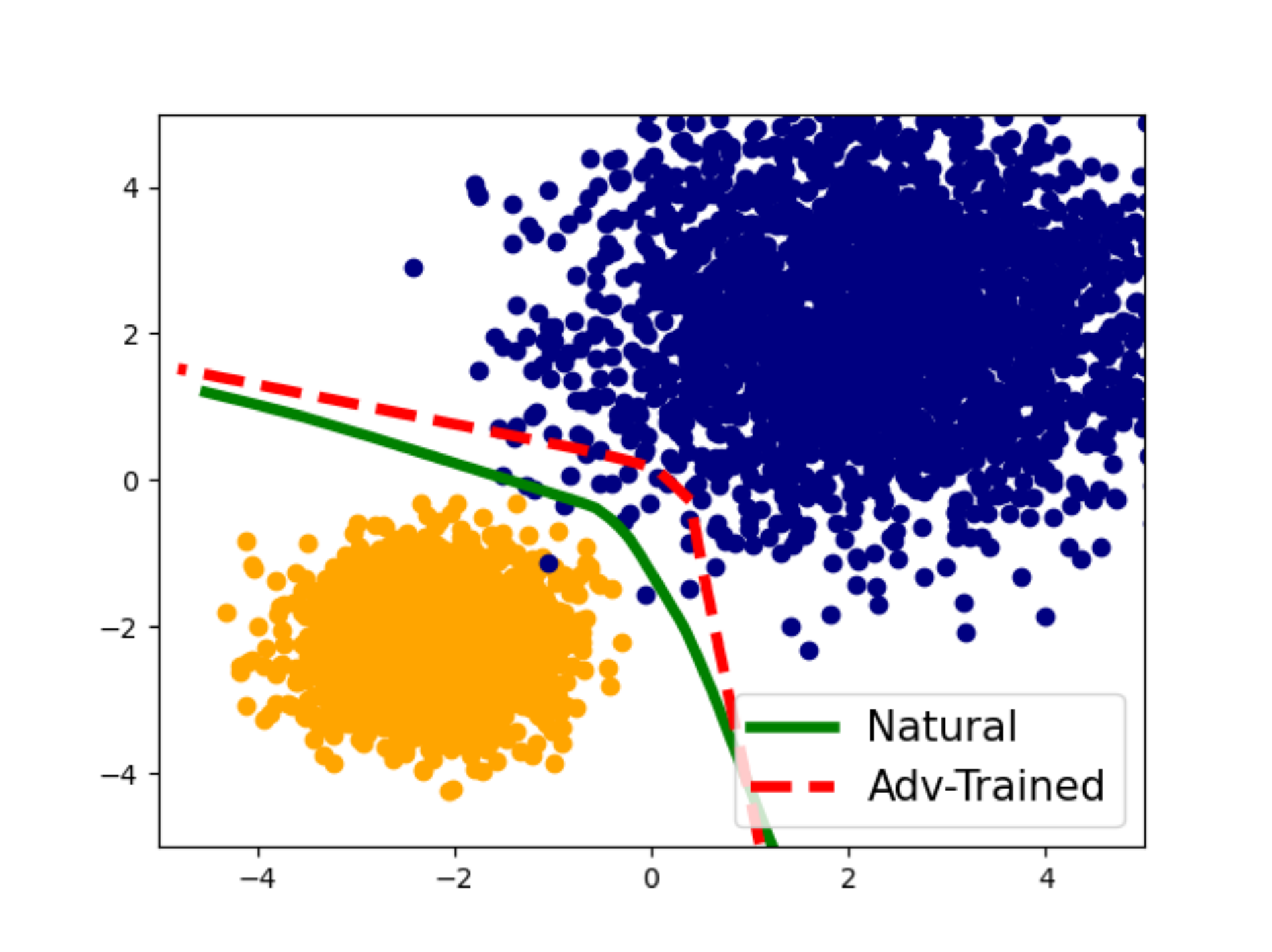}
\end{minipage}
}
\caption{Logistic Regression and Multi-perceptron classifiers (natural and robust) on simulated binary data in Eq.~\ref{eq:data_dist1}.}
\label{fig:2dim}
\end{figure}

In Figure~\ref{fig:2dim},we give an illustration for the theories in a 2-dim space to show the effect of adversarial training. 
Particularly, we sampled Gaussian data from class ``+1'' (blue) which centered on $\theta = (2,2)$, and has variance $\sigma_{+1}^2 = 2$, as well as class ``-1'' (yellow) which centered on $-\theta = (-2,-2)$, and has variance $\sigma_{-1}^2 = 1$.
We apply a logistic regression classifier for natural training (green line) and adversarial training (with perturbation bound $\epsilon = 0.5$, red line). From Figure~\ref{fig:2dim}, we get the consistent results with our theories: (1) there are more samples in the class ``+1'' than class ``-1'' which are wrongly predicted by a natural classifier; (2) adversarial training will further exacerbate the disparity by moving decision boundary closer to the harder class ``+1''. 
In Figure~\ref{fig:2dim} (right), we also show the results for a non-linear MLP classifier, where we can make similar observations. These theoretical understandings about adversarial training suggest that adversarial training will naturally bring unequal effect to different classes in the data distribution and consequently cause severe fairness issues.

\section{Fair Robust Learning (FRL)}~\label{sec:method}
The observations from both preliminary experiments and theoretical understandings suggest that the fairness issues in adversarial training can be a general and serious concern. In this section, we first discuss fairness requirements that an optimal robust model should satisfy, and then introduce algorithms to achieve these objectives.

\subsection{Objective of Robust Fairness}
In this work, we desire an optimal robust model that can achieve the parity of both standard prediction accuracy and adversarial robustness among each class $y\in Y$:
\begin{itemize}
  \setlength\itemsep{0em}
    \item \textbf{\textit{Equalized Accuracy}}: one classifier $f$'s standard error is statistically independent of the ground truth label $y$: $\text{Pr.}(f(x) \neq y | Y = y) \approx \text{Pr.}(f(x) \neq y)$ for all $y\in Y$.
    \item \textbf{\textit{Equalized Robustness}}: one classifier $f$'s robust error is statistically independent of the ground truth label $y$: $\text{Pr.}(\exists~ ||\delta||\leq\epsilon, f(x+\delta) \neq y | Y=y) \approx \text{Pr.}(\exists ~||\delta||\leq\epsilon, f(x+\delta) \neq y)$ for all $y\in Y$.
\end{itemize}
The notion of ``Equalized Accuracy'' is well studied in traditional fairness research~\citep{buolamwini2018gender,zafar2017fairness} which desires the model to provide equal prediction quality for different groups of people. The ``Equalized Robustness'' is our new desired ``fairness'' property for robustly trained models. For every class, the model should provide equal robustness and safety to resist adversarial attacks. Therefore, the robust model will have high overall safety but no obvious ``weakness''.

Faced with the fairness objectives mentioned above, in our work, we propose a Fair Robust Learning~(FRL) strategy to train robust models that have equalized accuracy and robustness performance for each class. Formally, we aim to train a classifier $f$ to have minimal overall robust error ($\mathcal{R}_\text{rob}(f)$), as well stressing $f$ to satisfy a series of fairness constraints as:
\begin{align}
\label{equ:constrain1}
\small
\begin{split}
&\underset{f}{\text{minimize}} ~~~~  \mathcal{R}_{\text{rob}}(f) \\
\text{s.t.} ~~~~ &
\begin{cases}
\mathcal{R}_{\text{nat}}(f,i) - \mathcal{R}_{\text{nat}}(f) \leq \tau_1& \\
\mathcal{R}_{\text{rob}}(f,i) - \mathcal{R}_{\text{rob}}(f) \leq \tau_2&
\text{ for each }i\in Y
\end{cases}
\end{split}
\end{align}
where $\tau_1$ and $\tau_2$ are small and positive predefined paramters. The constraints in Eq.~\ref{equ:constrain1} restrict the model's error for each class $i\in Y$ (both standard error $\mathcal{R}_\text{nat}(f,i)$ and robust error $\mathcal{R}_\text{rob}(f,i)$) should not exceed the average level ($\mathcal{R}_\text{nat}(f)$ and $\mathcal{R}_\text{rob}(f)$) by a large margin. Thus, there will be no obvious worst group in the whole dataset. One thing to note is that in Eq~\ref{equ:constrain1}, the robust error is always strongly related to the standard error~\cite{zhang2019theoretically, tsipras2018robustness} (see Eq.~\ref{equ:tradeoff}). Thus, during the debiasing process, we could have a twisted influence on the class $i$'s standard and robust errors. For example, if we apply some importance weighting methods to upweight the cost of $\mathcal{R}_{\text{rob}}(f,i)$, we also implicitly upweight the cost for $\mathcal{R}_{\text{nat}}(f,i)$. Therefore, we propose to separate the robust error into the sum of \textit{standard error} and \textit{boundary error} inspired by~\cite{zhang2019theoretically} as:
\begin{align}
\label{equ:tradeoff}
\small
\begin{split}
    &\mathcal{R}_{\text{rob}}(f,i)\\ 
    =& \text{Pr.}\{\exists \delta, \text{  s.t.  } f(x+\delta)\neq y| y = i\} \\
    =& \text{Pr.}\{f(x)\neq y| y = i\} + \text{Pr.}\{\exists \delta, f(x+\delta) \cdot f(x) \leq 0| y = i\}\\
    =& \mathcal{R}_{\text{nat}}(f,i) + \mathcal{R}_{\text{bndy}}(f,i)
\end{split}
\end{align}
where $\mathcal{R}_\text{bndy}(f,i) =  \text{Pr.}\{\exists\delta, f(x+\delta) \cdot f(x) \leq 0| y = i\}$ represents the probability that a sample from class $i$ lies close to the decision boundary and can be attacked. By separating the standard error and boundary error during adversarial training, we are able to independently solve the unfairness of both standard error and boundary error. Formally, we have the training objective as:
\begin{align}
\label{equ:constrain2}
\small
\begin{split}
\underset{f}{\text{minimize}}~~~~ & \mathcal{R}_{\text{nat}}(f) + \mathcal{R}_{\text{bndy}}(f)\\
\text{s.t.} ~~~~ 
&\begin{cases}
\mathcal{R}_{\text{nat}}(f,i) - \mathcal{R}_{\text{nat}}(f) \leq \tau_1\\
\mathcal{R}_{\text{bndy}}(f,i) - \mathcal{R}_{\text{bndy}}(f) \leq \tau_2
\text{ for each }i\in Y
\end{cases}
\end{split}
\end{align}
During training to optimize the boundary errors, we borrow the idea from \cite{zhang2019theoretically}, which minimizes the KL-divergence between the output logits of clean samples and their adversarial samples. In the following subsections, we explore effective methods to solve the problem in Eq.~\ref{equ:constrain2}.

\subsection{Reweight for Robust Fairness}
In order to solve the fair robust training problem in Eq.~\ref{equ:constrain2}, we first follow the main pipeline from traditional machine learning debiasing works such as~\citep{agarwal2018reductions, zafar2017fairness}, which reduce the problem in Eq.~\ref{equ:constrain1} into a series of \textit{Cost-sensitive} classification problems and continuously penalize the terms which violate the fairness constraints. We begin by introducing Lagrange multipliers $\phi = (\phi^i_\text{nat}, \phi^i_\text{bndy})$ (non-negative) for each constraint in Eq.~\ref{equ:constrain1} and form the Lagrangian:
\begin{align}
\label{equ:training}
\small
\begin{split}
L(f,\phi) = &\mathcal{R}_{\text{nat}}(f) +  \mathcal{R}_{\text{bndy}}(f)\\ &+\sum_{i=1}^{Y}{\phi^i_\text{nat}(\mathcal{R}_{\text{nat}}(f,i) - \mathcal{R}_{\text{nat}}(f) -  \tau_1)^+}\\
&+\sum_{i=1}^{Y}{\phi^i_\text{bndy}(\mathcal{R}_{\text{bndy}}(f,i)- \mathcal{R}_{\text{bndy}}(f) -  \tau_2)^+}
\end{split}
\end{align}
Thus, the problem in Eq.~\ref{equ:constrain2} equals to solving the max-min game between two rivals $f$ and $\phi$ as:
\begin{align}
\label{equ:minmax}
    \max_{\phi_{\text{nat}}, \phi_{\text{rob}}\geq 0} \min_{f} ~~L(f,\phi).
\end{align}
We present the details for solving Eq.~\ref{equ:minmax} in Algorithm~\ref{alg:debias}. During the training process, we start from a pre-trained robust model and we test its class-wise standard / boundary errors on a separated validation set (step 5 and 6). We check whether the current model $f$ violates some constraints in Eq.~\ref{equ:constrain2}. For example, if there exists a class ``i'', whose standard error is higher than the average by a larger margin: $\mathcal{R}_{\text{nat}}(f,i) - \mathcal{R}_{\text{nat}}(f) - \tau_1> 0$. We will adjust its multiplier term $\phi^i_{\text{nat}}$ according to the extent of violation (step 7). As a result, we increase the training weight for the standard error $\mathcal{R}_\text{nat}(f,i)$ for the class $i$. We also adjust the multiplier for boundary errors at the same time (step 8). Next, fixing the multipliers $\phi$, the algorithm will solve the inner-minimization to optimize the model $f$. By repeating the steps 4-9, the model $f$ and Lagrangian multiplier $\phi$ will be alternatively updated to achieve the equilibrium until we finally reach an optimal model that satisfies the fairness constraints. Note that we denote the framework with the proposed Reweight strategy as FRL (Reweight).

\begin{algorithm}[h]
\begin{algorithmic}[1]
\setstretch{1.3}
\STATE \textbf{Input:} Fairness constraints specified by $\tau_1>0$ and $\tau_2>0$, test time attacking radius $\epsilon$ and hyper-param update rate $\alpha_1, \alpha_2$ \\
\STATE \textbf{Output:} A fairly robust neural network $f$
\STATE Initialize network with a pre-trained robust model \\
Set $\phi^i_\text{nat} = 0$, $\phi^i_\text{bndy} = 0$ and $\phi = (\phi_\text{nat},\phi_\text{bndy})$, 
\REPEAT
\STATE $\mathcal{R}_\text{nat}(f), \mathcal{R}_\text{nat}(f,i) = \text{EVAL}(f)$ 
\hfill 
\STATE $\mathcal{R}_\text{bndy}(f), \mathcal{R}_\text{bndy}(f,i) = \text{EVAL}(f,\epsilon)$ 
\hfill 
\STATE $\phi^i_\text{nat} = \phi^i_\text{nat} + \alpha_1\cdot(\mathcal{R}_\text{nat}(f,i)-\mathcal{R}_\text{nat}(f)  - \tau_1)$ \hfill
\STATE$\phi^i_\text{bndy} = \phi^i_\text{bndy} + \alpha_2\cdot(\mathcal{R}_\text{bndy}(f,i)-\mathcal{R}_\text{bndy}(f)  - \tau_2)$
\hfill
\STATE $f \leftarrow{\text{TRAIN}(f, \phi, \epsilon)}$ \hfill
\UNTIL{Model $f$ satisfies all constraints}
\caption{The Fair Robust Learning (FRL) Algorithm}
\label{alg:debias}
\end{algorithmic}
\end{algorithm}

\subsection{ReMargin for Robust Fairness}~\label{sec:remargin}
Though upweighting the cost for $\mathcal{R}_{\text{nat}}(f,i)$ has the potential to help penalize large $\mathcal{R}_\text{nat}(f,i)$ and improve the standard performance for worse groups. However, only upweighting one class's boundary error's cost could not succeed to fulfill the goal to help decrease its boundary error. In Section ~\ref{sec:ablation}, we show this fact in PGD adversarial training on CIFAR10, where we find that even we give a large weight ratio for a class's boundary error, it is not sufficient to reduce the boundary error for this class. Thus, we cannot achieve to mitigate the boundary error disparity and robust error disparity. 

To solve this problem, we propose an alternative strategy by enlarging the perturbation margin $\epsilon$.
It is evident from some existing works~\citep{tramer2020fundamental, ding2018max} that increasing the margin $\epsilon$ during adversarial training can effectively improve model's robustness against attacks under the current intensity $\epsilon$. Therefore, enlarging the adversarial margin $\epsilon$ when generating adversarial examples during training specifically for the class $i$ has the potential to improve this class's robustness and reduce the large boundary error $\mathcal{R}_\text{bndy}(f,i)$. In this work, we define this strategy as FRL~(Remargin). The FRL~(Remargin) resembles the procedure in Algorithm~\ref{alg:debias}, except for the step 7, where we instead update the adversarial margin for the boundary errors. Specifically, we change the adversarial margin $\epsilon^i$ of the class ``$i$'' as follows:
\begin{align}
    \epsilon^i = \epsilon^i \cdot \exp \left(\alpha_2^* (\mathcal{R}_\text{bndy}(f,i) - \tau_2))\right)
\end{align}
Besides FRL~(Reweight) and FRL~(Remargin), we can combine Reweight and Remargin, where we jointly update the weight of the boundary errors and change the margin.

\section{Experiment}\label{sec:experiments}
In this section, we present the experimental results to validate the effectiveness of the proposed framework (FRL) on building fairly robust DNN models. We implement and compare our proposed three strategies (i.e., Reweight, Remargin and Reweight+Remargin) on real-world data and discuss their effectiveness. The implementation of the proposed algorithms can be found via \url{ https://drive.google.com/open?id=1IjE11VpV5CGXx633Nef3McgVOoRGMK9x}. 

\subsection{Experimental Settings} \label{sec:setup}
We conduct our experiments on benchamrk adversarial learning datasets, including CIFAR10~\citep{krizhevsky2009learning} and SVHN dataset~\cite{netzer2011reading}. For both datasets, we study the algorithms under the model architectures PreAct-ResNet18 and WRN28. We only present the results of PreAct-ResNet18 in this section and we leave the results of WRN28 in Appendix~\ref{appendix:WRN28}. As baselines, we present the original performance of two popular adversarial training algorithms~\citep{madry2017towards, zhang2019theoretically}, and a debiasing method which is inherited from~\citep{agarwal2018reductions}. It is a traditional debiasing technique and we directly apply it to upweight the cost of the class with the highest robust error in the training data. Other existing unfairness debiasing methods, such as~\cite{zafar2017fairness, zhang2018mitigating} are not included in our experiment, because they are not proposed for deep learning models and have similar ideas with~\cite{agarwal2018reductions} to reweight the costs during training.

In our implementation for FRL methods, during the training process, we split the training sets to get validation sets with 300 samples in each class to help us adjust the hyperparameters. For each variant of our proposed FRL method, we pre-define the model to achieve fairness constraints to satisfy that both $\tau_1$ and $\tau_2$ are not larger than 5\% or 7\%. In the training, we start FRL from a pre-trained robust model (such as a PGD-adversarial training), and run FRL with model parameter learning rate 1e-3 and hyperparameter learning rate $\alpha_1 = \alpha_2 = 0.05$ in the first 40 epochs. Then we decay the model parameter learning rate and the hyperparameter learning rate by $0.1$ every 40 epochs. During the evaluation phase, we report each trained model's average \textit{standard error} rate, \textit{boundary error} rate and \textit{robust error} rate, as well as the worst intraclass error rate. Note that the boundary and robust errors are calculated by the PGD attack under $l_\infty$-norm 8/255.


\begin{table*}[t]
\centering
\caption{Average \& worst-class standard error, boundary error and robust error for various algorithms on CIFAR10.}
\resizebox{0.85\textwidth}{!}
{
\begin{tabular}{c |c c |c c |c c} 
\hline
& \textbf{Avg. Std.} & \textbf{Worst Std.} & \textbf{Avg. Bndy.} & \textbf{Worst Bndy.} & \textbf{Avg. Rob.} & \textbf{Worst Rob.}  \\  
\hline\hline
\textbf{PGD Adv. Training}& 15.5 & 33.8 & 40.9 & 55.9 & 56.4 & 82.7\\ 
\textbf{TRADES$(1/\lambda = 1)$} &\textbf{14.6} & 31.2 & 43.1 & 64.6 & 57.7 & 84.7\\ 
\textbf{TRADES$(1/\lambda = 6)$} &19.6 & 39.1 & \textbf{29.9} & 49.5 & \textbf{49.3} & 77.6\\ 
\textbf{Baseline Reweight} &19.2&28.3&39.2&53.7&58.2&80.1\\ 
\hline
\textbf{FRL(Reweight, 0.05)} &16.0 & \textbf{22.5} & 41.6 & 54.2 & 57.6 & 73.3\\ 
\textbf{FRL(Remargin, 0.05)}&16.9 & 24.9 & 35.0 & 50.6 & 51.9 & 75.5\\ 
\textbf{FRL(Reweight+Remargin, 0.05)} &17.0 & 26.8 & 35.7 & \textbf{44.5} & 52.7 & \textbf{69.5}\\ 
\hline
\textbf{FRL(Reweight, 0.07)} &16.1 & 23.8 & 38.4 & 55.2 & 54.0 & 75.2\\ 
\textbf{FRL(Remargin, 0.07)}&16.9 & 26.0 & 37.4 & 51.6 & 53.5 & 75.1\\ 
\textbf{FRL(Reweight+Remargin, 0.07)} &17.1 & 26.7 & 36.7 & 48.3 & 53.8 & 70.2\\ 
\hline
\end{tabular}}
\label{tab:debias}
\end{table*}

\begin{table*}[t]
\centering
\caption{Average \& worst-class standard error, boundary error and robust error for various algorithms on SVHN. }
\resizebox{0.85\textwidth}{!}
{
\begin{tabular}{c |c c |c c |c c} 
\hline
& \textbf{Avg. Std.} & \textbf{Worst Std.} & \textbf{Avg. Bndy.} & \textbf{Worst Bndy.} & \textbf{Avg. Rob.} & \textbf{Worst Rob.}  \\  
\hline\hline
\textbf{PGD Adv. Training}& 9.4 & 19.8 & 37.0 & 53.9 & 46.4 & 73.7\\ 
\textbf{TRADES$(1/\lambda = 1)$} & 9.9 & 18.6 & 39.1 & 60.6 & 48.0 & 78.3\\ 
\textbf{TRADES$(1/\lambda = 6)$} &10.5 & 23.4 & \textbf{32.5} & 52.5 & \textbf{43.1} & 76.6\\ 
\textbf{Baseline Reweight} &8.8& 17.4&39.3&54.7&48.2&72.1\\ 
\hline
\textbf{FRL(Reweight, 0.05)} &7.9 & \textbf{13.3} & 38.2 & 56.4 & 46.1 & 69.7\\ 
\textbf{FRL(Remargin, 0.05)}&9.2 & 17.4 & 39.7 & \textbf{49.6} & 48.9 & 67.0\\ 
\textbf{FRL(Reweight+Remargin, 0.05)} &\textbf{7.7} & 12.8 & 36.2 & 51.2 & 43.9 & \textbf{64.0}\\ 
\hline
\textbf{FRL(Reweight, 0.07)} &8.0 & 13.6 & 37.2 & 54.2 & 45.0 & 67.8\\ 
\textbf{FRL(Remargin, 0.07)}&8.5 & 14.2 & 36.9 & 50.6 & 45.5 & 64.8\\ 
\textbf{FRL(Reweight+Remargin, 0.07)} &8.3 & 15.4 & 36.7 & 51.4 & 45.0 & 64.9\\ 
\hline
\end{tabular}}
\label{tab:debias2}
\end{table*}

\subsection{Fairness Performance} 
 
Table~\ref{tab:debias} shows each algorithm's performance on the CIFAR10 dataset, including each variant of FRL under the fairness constraints $\tau_1 = \tau_2 = 5\%$ and  $\tau_1 = \tau_2 = 7\%$.
From the experimental results, we can see that all FRL algorithms reduce the worst-class standard errors and robust errors under different degrees. FRL~(Reweight) has the best performance to achieve the minimal ``worst-class'' standard error. Compared to vanilla methods, it has around $10\%$ reduction to the worst class's standard error. However, it cannot equalize the boundary errors adequately. Alternatively, FRL~(Remargin) is more effective than FRL~(Reweight) to decrease the worst class boundary error. Furthermore, their combination FRL~(Reweight + Remargin) is the most effective way to reduce the worst-class boundary and worst-class robust errors. It can accomplish to train a robust classifier with the worst-class robust error around $70\%$, which is $10\sim15\%$ lower than vanilla adversarial training methods. The baseline method (Baseline Reweight)~\citep{agarwal2018reductions} can only help decrease the worst-class standard error but cannot equalize boundary errors or robust errors. These results suggest that the FRL method can mitigate the fairness issues by improving the model's standard and robustness performance on the worst classes.

Notably, from the results in Table~\ref{tab:debias}, we find that those methods, which use the strategy ``Remargin'', usually have a $1\sim2\%$ larger average standard error, compared to the PGD Adversarial Training or TRADES ($1/\lambda = 1$). This is because ``Remargin'' increases the perturbation margin $\epsilon$ for some classes during training. According to the existing works ~\cite{tramer2020fundamental, shafahi2019adversarial}, using a large perturbation margin might degrade the model's standard performance generalization. Thus, in this work we limit every training sample's perturbation margin does not exceed $16/255$ to guarantee that the model gives an acceptable average standard performance. On the other hand, because the ``Remargin'' strategies increase the perturbation margin, their average boundary errors are smaller than the vanilla methods. As a result, the average robust errors are comparable with the vanilla methods or even have $2\sim3\%$ improvement.

In Table~\ref{tab:debias2}, we present the experimental results on the SVHN dataset. We have similar observations as those on CIFAR10. FRL (Reweight) can achieve the minimal worst-class standard error and FRL~(Reweight+Remargin) gives the minimum worst-class robust error. One intriguing fact is that FRL methods, including those which use ``Remargin'', do not cause an increase of the average standard error. Instead, each FRL method results in  $1\sim2\%$ decrease of the average standard error. This might be due to the reason that these methods help improve the worst-class standard error by a large margin.

\subsection{Ablation Study}\label{sec:ablation}
From the results in Table~\ref{tab:debias}, we find that FRL~(Reweight) is not effective to improve the worst-class boundary error. As a result, it cannot sufficiently equalize the robustness performance between classes. In this subsection, we study the potential reasons that cause this fact. To have a closer look at this issue, we implement adversarial training on CIFAR10 dataset in two groups of experiments following the basic settings described in Section~\ref{sec:setup}. For each group, we only upweight a chosen class's boundary error by different ratios (from 1.0 to 4.5 times of other classes) and keep other classes' weights fixed. Then, for each model, we present this model's standard / robustness performance for this class in Figure~\ref{fig:rw} (left). 
As a comparison, we also show the cases where we increase the perturbation margin (1.0 to 2.5 times of the original margin) (Figure~\ref{fig:rw} (right)). 
\begin{figure}[h]
\subfloat[Upweight Boundary Error]{\label{clean}
\begin{minipage}[c]{0.24\textwidth}
\centering
\includegraphics[width = 1\textwidth]{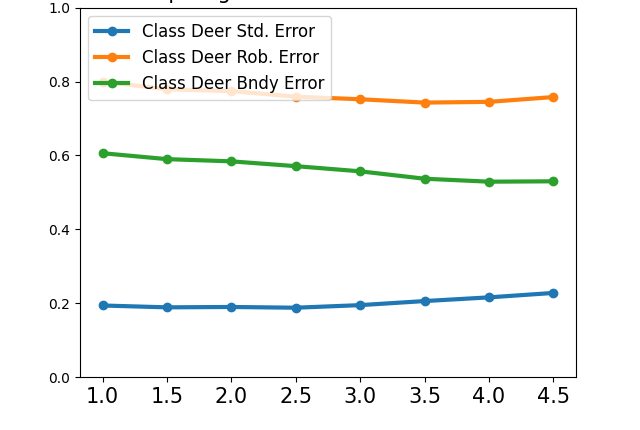}
\end{minipage}
}
\subfloat[Increase Perturbation Margin]{
\begin{minipage}[c]{0.24\textwidth}
\centering
\includegraphics[width = 1\textwidth]{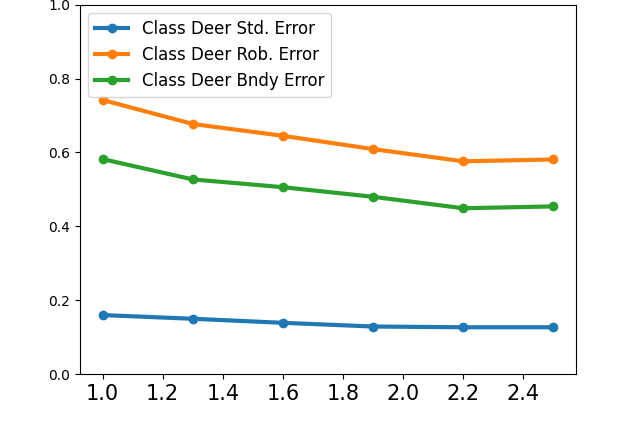}
\end{minipage}
}
\caption{The effect of upweighting on  boundary error (left) and standard error (right) for the class ``deer''.}
\label{fig:rw}
\end{figure}
From Figure~\ref{fig:rw}, we find that when we increase the weight only for the class ``deer'', it results in the boundary error for this class to decrease but also increasing its standard error. Thus, reweight only acts to leverage the inner-class boundary error and standard error. It cannot improve the robustness of this class over other classes to solve the fairness relationship between classes.
As a contrast, increasing the margin will not increase the standard error while it effectively decreases this class's boundary error and robust error. Therefore, from the results in Table~\ref{tab:debias} and \ref{tab:debias2}, we observe that ``Remargin'' based methods can successfully improve the worst-class robustness performance.

\section{Related Work}

\noindent \textbf{Adversarial Attacks and Adversarial Training.}
The existence of adversarial attacks~\citep{goodfellow2014explaining, szegedy2013intriguing, carlini2017towards} causes huge concerns when people adopt machine learning models in various application domains~\citep{xu2019adversarial, jin2020adversarial}. As countermeasures against adversarial examples, adversarial training (robust optimization) algorithms~\citep{goodfellow2014explaining, madry2017towards, zhang2019theoretically, shafahi2019adversarial, zhang2019you} are formulated as a min-max problem that directly minimize the model's risk on the adversarial samples. Another mainstream of defense methods are certified defense, which aims to provide provably robust DNNs under $l_p$ norm bound~\citep{wong2018provable, cohen2019certified} and guarantee the robustness. 


\noindent \textbf{Fairness in Machine Learning \& Imbalanced Dataset.}
Fairness issues recently draw much attention from the community of machine learning. These issues can generally divided into two categorizations: (1) prediction outcome disparity~\citep{zafar2017fairness}; and (2) prediction quality disparity~\citep{buolamwini2018gender}. Unlike existing works, this work is the first study the unfairness issue in the adversarial setting. We also mention the imbalanced data learning problem~\citep{he2009learning, lin2017focal} as one related topic of our work. Since in our work, (i.e..~Figure \ref{fig:fair}), we show that the prediction performance differences are indeed existing between different classes. This phenomenon is also well studied in imbalanced data problems or long-tail distribution learning problems~\citep{wang2017learning} where some classes have much fewer training samples than others. However, in our case, we show that this unfairness problem can generally happen in balanced datasets, so it desires new scopes and methods for further study.

\noindent \textbf{Fairness in Robust Learning} A parallel and independent work \cite{nanda2020fairness} also figures out that the phenomenon of class-wise unequal robustness can happen for many deep learning tasks in the wild. While our work is more focused on adversarial training algorithms and we argue that adversarial training methods can have the property to cause these fairness phenomena. In addition, our work also discusses various potential mitigation methods to achieve more balanced robustness for adversarial training methods.

\section{Conclusion}

In this work we first empirically and theoretically uncover one property of adversarial training algorithms: it can cause serious disparity for both standard accuracy and adversarial robustness between different classes of the data. As the first attempt to mitigate the fairness issues from adversarial training, we propose the Fair Robust Learning (FRL) framework. We validate the effectiveness of FRL on benchmark datasets. In the future, we want to examine if the fairness issue can be observed in other types of defense methods.


\balance
\bibliography{sample}
\bibliographystyle{icml2021}

\newpage
\appendix

\section{Appendix}

\subsection{More Results of Preliminary Studies}\label{app:pre}

In this section, we show more preliminary results on the robust fairness phenomenon of adversarial training in various settings. In addition to the results shown in Section~\ref{pre}, we present the results in the settings with one more architecture (WRN28), one more type of adversarial attack ($l_2$-norm attack), one more defense method (Randomized Smoothing) and one more dataset (SVHN). From all these settings, we observe the similar phenomenon as in Section~\ref{pre}, which show that the fairness phenomenon can be generally happening in adversarial training under different scenarios and can be a common concern during its application. Furthermore, we also present the detailed results as in Table~\ref{tab:diff}, to show the fact that adversarial training usually gives an unequal influence on different classes, which can be a reason that causes this fairness phenomenon.

In detail, for each dataset under PreAct ResNet18 architecture, for each adversarial training algorithm (including PGD-adversarial training~\cite{madry2017towards} and TRADES~\cite{zhang2017age}), we train the models following that as suggested by the original papers. We train the models for 200 epochs with learning rate 0.1 and decay the learning rate at the epoch 100 and 150 by factor 0.1. During the evaluation phase, we report the trained model's classwise \textit{standard error} rate and \textit{robust error} rate. In general settings without explicit mention, we study the models' robustness against $l_\infty$-norm adversarial attack under $8/255$, where we implement PGD attack algorithm for 20 steps for robustness evaluation.

\subsubsection{Robust Fairness in WRN28 in CIFAR10}
Figure~\ref{fig:fair_wrn28} presents robust fairness issues in CIFAR10 dataset under WRN28 models. Note that in Section~\ref{pre} we also presented the corresponding results under PreAct ResNet18 models in Figure~\ref{fig:fair}. We can observe the similar phenomenon about the robust fairness issues under both models. Moreover, as clear evidence of the unequal effect of adversarial training among different classes, in Table~\ref{tab:diff_complete} and Table~\ref{tab:diff_wrn28}, we compare the classwise standard error and robust error between natural training and PGD adversarial training. From the experimental results, we get the conclusion that adversarial training usually increases a larger error rate for the classes, such as ``dog'' and ``cat'', which originally have larger errors in natural training. Similarly, adversarial training will also give less help to reduce the robust errors for these classes.

\subsubsection{Robust Fairness in $l_2$-norm Adversarial Training}

Figure~\ref{fig:fair_l2} presents the robust fairness issues of adversarial training methods which target on $l_2$-norm attacks in CIFAR10 dataset. We further confirm the existence of robust unfairness in adversarial training methods. In Figure~\ref{fig:fair_l2}, we present the classwise standard errors and robust errors, which target on $l_2$-norm 0.5 adversarial attack. During the robustness evaluation, we implement PGD attack algorithm with step size 0.1 for 20 steps.

\subsubsection{Robust Fairness for Certified Defenses}

Certified defenses are another main type of effective defense strategies. Even though certified defenses do not train in the same way as traditional adversarial training methods, which train the models on the generated adversarial examples, they minimize the probability of the existence of adversarial examples near the input data. This process also implicitly minimizes the model's overall robust error. Thus, in this section we study whether this certified defense will have robust fairness issues. As a representative, we implement Randomized Smoothing~\cite{cohen2019certified}, which is one state-of-the-art methods to certifiably defense against $l_2$-norm adversarial attacks. In this experiment, we run Randomized Smoothing against $l_2$-norm 0.5 attacks in CIFAR10 dataset and report its class-wise certified standard error and certified robust error under different intensities. 

\begin{figure}[h]
    \centering
    \includegraphics[width = 0.75\linewidth]{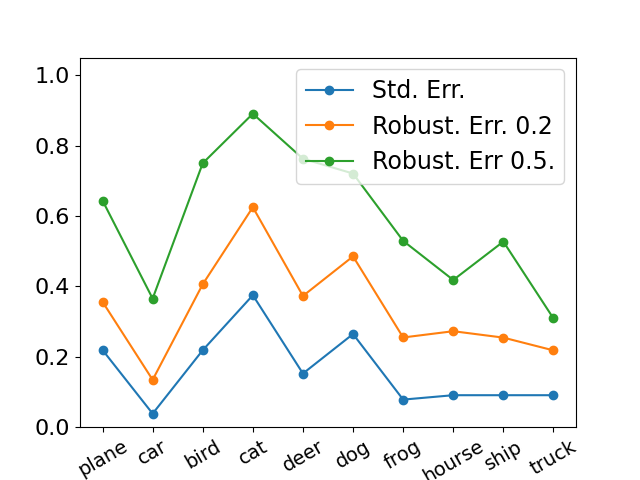}
    \caption{Randomized Smoothing on CIFAR10}
    \label{fig:multi_example}
\end{figure}
The results also suggest that the Randomized Smoothing certified defense method presents the similar disparity phenomenon as the traditional adversarial training methods. Moreover, it also preserves the similar classwise performance relationship, i.e., it has both high standard \& robustness error on the classes ``cat'' and ``dog'', but has relatively low errors on ``car'' and ``ship''. 
\subsubsection{Robust Fairness on SVHN Dataset}

Figure~\ref{fig:fair_resnet18} presents the robust fairness issues of adversarial training methods in SVHN dataset under PreAct ResNet18 model. From the experimental results, we also observe the strong disparity of classwise standard errors and robust errors, which do not exist in natural training. In particular, the classes ``3'' and ``8'' have the largest standard error in a naturally trained model. After adversarial training, these two classes also have the largest standard error increases among all classes, as well as the least robust error decreases. As a result, there is also a significant disparity of the standard / robustness performance among the classes. The full results are shown in Table~\ref{tab:diff_svhn}.

\begin{figure*}[h]
\subfloat[Natural Training]{\label{clean}
\begin{minipage}[c]{0.25\textwidth}
\centering
\includegraphics[width = 1\textwidth]{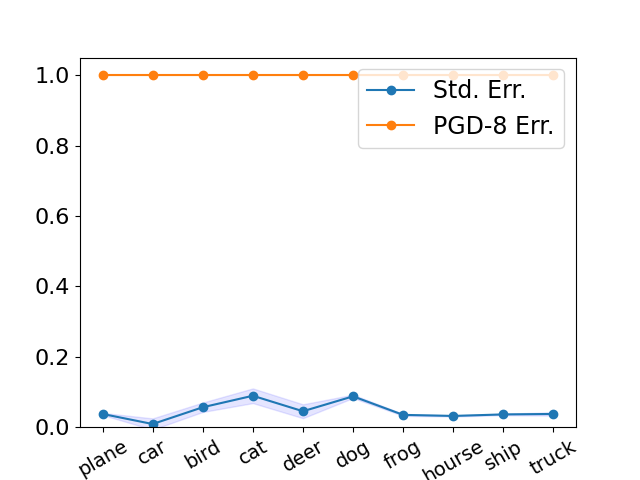}
\end{minipage}
}
\subfloat[PGD Adversarial Training]{ \label{pgd_2}
\begin{minipage}[c]{0.25\textwidth}
\centering
\includegraphics[width = 1\textwidth]{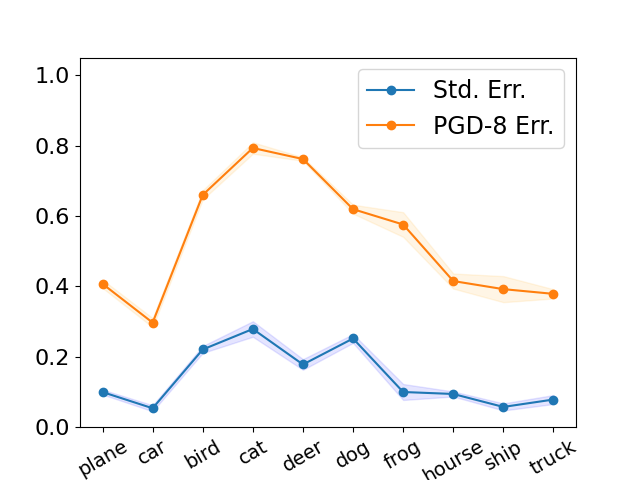}
\end{minipage}
}
\subfloat[TRADES ($1/\lambda = 1$)]{  \label{pgd_4}
\begin{minipage}[c]{0.25\textwidth}
\centering
\includegraphics[width = 1\textwidth]{figures/trades1.0_wrn28.png}
\end{minipage}
}
\subfloat[TRADES ($1/\lambda = 6$)]{ \label{pgd_8}
\begin{minipage}[c]{0.25\textwidth}
\centering
\includegraphics[width = 1\textwidth]{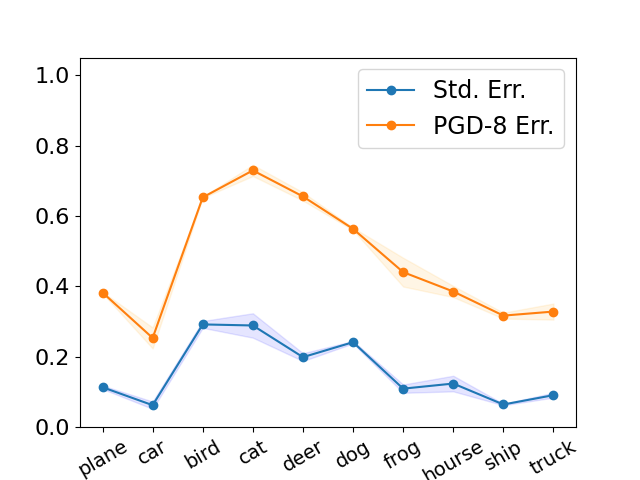}
\end{minipage}
}
\caption{The class-wise performance of natural / adversarial training (target on $l_\infty$-norm 8/255 attack) on CIFAR10 under WRN28.}
\label{fig:fair_wrn28}
\end{figure*}

\begin{figure*}[h!]
\centering
\subfloat[Natural Training]{ \label{pgd_2}
\begin{minipage}[c]{0.25\textwidth}
\centering
\includegraphics[width = 1\textwidth]{figures/clean_train.png}
\end{minipage}
}
\subfloat[PGD Adversarial Training]{\label{clean}
\begin{minipage}[c]{0.25\textwidth}
\centering
\includegraphics[width = 1\textwidth]{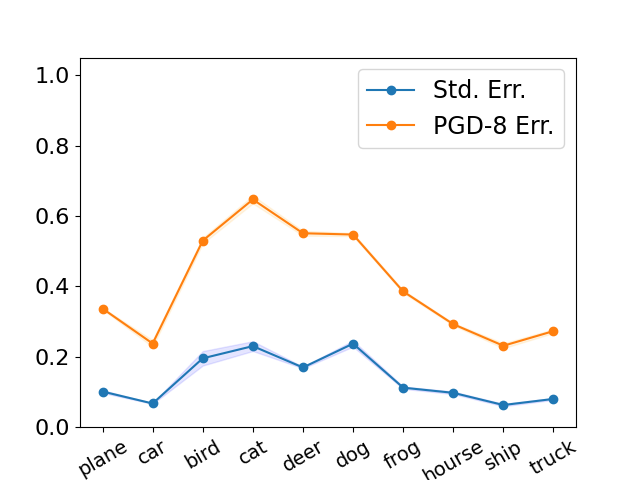}
\end{minipage}
}
\subfloat[TRADES ($1/\lambda = 6$)]{ \label{pgd_2}
\begin{minipage}[c]{0.25\textwidth}
\centering
\includegraphics[width = 1\textwidth]{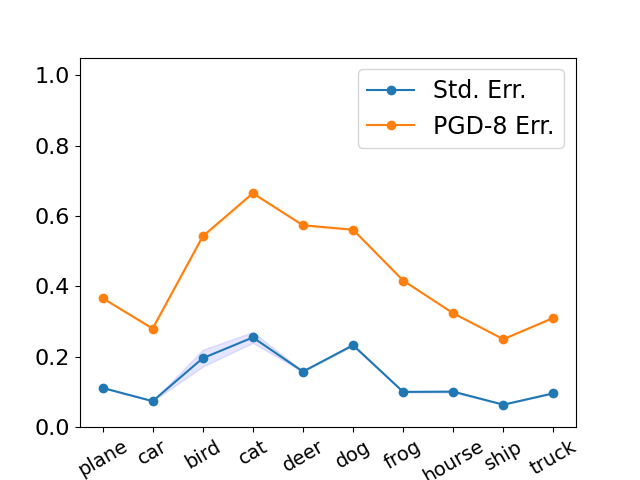}
\end{minipage}
}
\subfloat[TRADES ($1/\lambda = 6$)]{  \label{pgd_4}
\begin{minipage}[c]{0.25\textwidth}
\centering
\includegraphics[width = 1\textwidth]{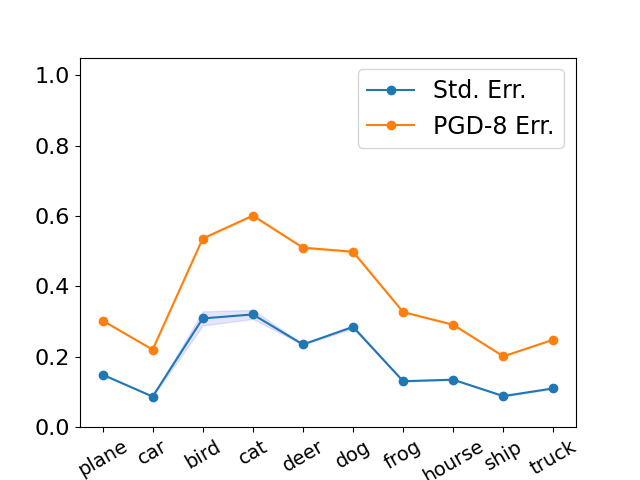}
\end{minipage}
}
\caption{The class-wise performance of natural / adversarial training (target on $l_2$-norm 0.5 attack) on CIFAR10 under PreActResNet18}
\label{fig:fair_l2}
\end{figure*}

\begin{figure*}[h!]
\subfloat[Natural Training]{\label{clean}
\begin{minipage}[c]{0.25\textwidth}
\centering
\includegraphics[width = 1\textwidth]{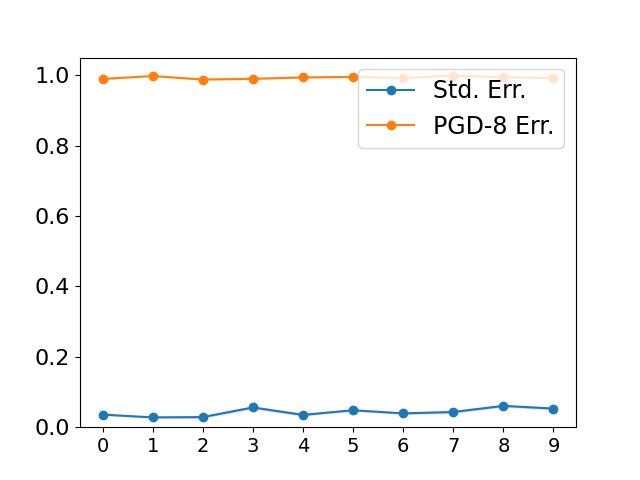}
\end{minipage}
}
\subfloat[PGD Adversarial Training]{ \label{pgd_2}
\begin{minipage}[c]{0.25\textwidth}
\centering
\includegraphics[width = 1\textwidth]{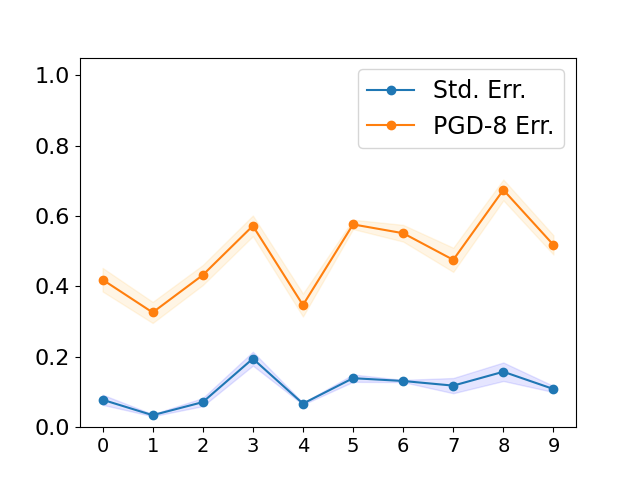}
\end{minipage}
}
\subfloat[TRADES ($1/\lambda = 1$)]{  \label{pgd_4}
\begin{minipage}[c]{0.25\textwidth}
\centering
\includegraphics[width = 1\textwidth]{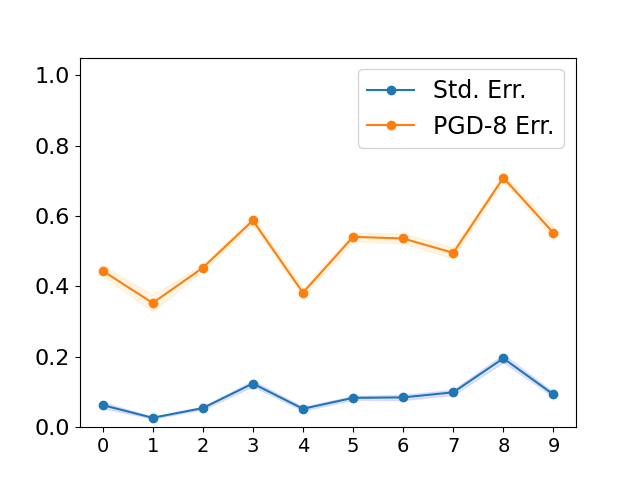}
\end{minipage}
}
\subfloat[TRADES ($1/\lambda = 6$)]{ \label{pgd_8}
\begin{minipage}[c]{0.25\textwidth}
\centering
\includegraphics[width = 1\textwidth]{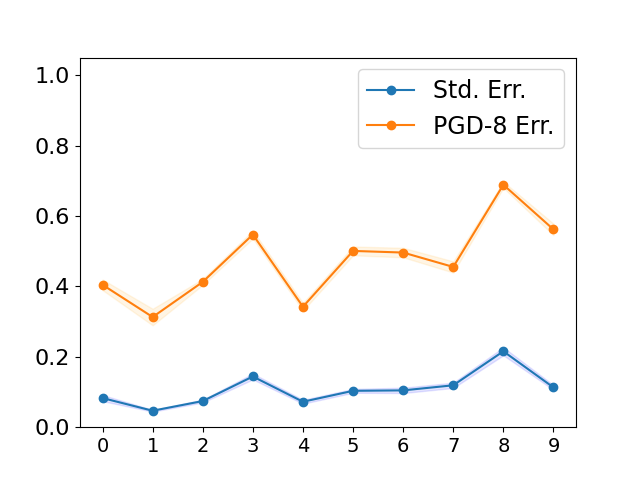}
\end{minipage}
}
\caption{The class-wise performance of natural / adversarial training (target on $l_\infty$-norm 8/255 attack) on SVHN under PreActResNet18.}
\label{fig:fair_resnet18}
\end{figure*}

\begin{table*}[h]
\centering
\caption{The Changes of Standard \& Robust Error in Natural \& Adversarial Training in CIFAR10 on PreAct ResNet18.}
\scalebox{0.9}{
\begin{tabular}{c |c c c c c c c c c c} 
\hline
Std. Error & plane & car & bird & cat & deer & dog & frog & horse & ship  &truck \\  
\hline\hline
\textbf{Natural Training} & 4.0 & 1.8 & 6.4 & 11.3 & 6.1 & 10.0 & 5.1 & 5.2 & 3.5 & 4.3\\ 
\textbf{PGD Adv. Training} & 11.7 & 6.1 & 23.3 & 34.8 & 20.8 & 26.9 & 12.6 & 9.8 & 6.4 & 9.5\\ 
\textbf{Diff. (Adv. - Nat.)} & 7.7 & 4.3 & 16.9 & 23.5 & 14.6 & 16.9 & 7.5 & 4.6 & 2.9 & 5.2\\ 
\hline
\end{tabular}}

\vspace{0.2cm}
\scalebox{0.9}{
\begin{tabular}{c |c c c c c c c c c c} 
\hline
Rob. Error & plane & car & bird & cat & deer & dog & frog & horse & ship  &truck \\  
\hline\hline
\textbf{Natural Training} & 100  & 100  & 100  & 100  & 100  & 100  & 100  & 100  & 100  & 100 \\
\textbf{PGD Adv. Training} & 44.9 & 34.3 & 68.4 & 82.7 & 74.7 & 66.4 & 51.5 & 47.0 & 40.8 & 42.3\\ 
\textbf{Diff. (Adv. - Nat.)} & -55.1 & -65.7 & -31.6 & -17.3 & -25.3 & -33.5 & -48.5 & -53.0 & -59.2 & -57.7\\ 
\hline
\end{tabular}}
\label{tab:diff_complete}
\end{table*}
\begin{table*}[h!]
\centering
\caption{The Changes of Standard \& Robust Error in Natural \& Adversarial Training in CIFAR10 on WRN28.}
\scalebox{0.9}{
\begin{tabular}{c |c c c c c c c c c c} 
\hline
Std. Error & plane & car & bird & cat & deer & dog & frog & horse & ship  &truck \\  
\hline\hline
\textbf{Natural Training} & 3.9 & 1.9 & 6.4 & 10.0 & 6.2 & 9.0 & 5.0 & 4.6 & 3.3 & 4.2\\
\textbf{PGD Adv. Training} & 10.0 & 4.9 & 21.4 & 24.6 & 17.4 & 26.2 & 12.4 & 9.4 & 6.3 & 9.2\\ 
\textbf{Diff. (Adv. - Nat.)} & 6.1 & 3.0 & 15.0 & 14.6 & 11.2 & 17.2 & 7.4 & 4.8 & 3.0 & 5.0\\ 
\hline
\end{tabular}}

\vspace{0.2cm}
\scalebox{0.9}{
\begin{tabular}{c |c c c c c c c c c c} 
\hline
Rob. Error & plane & car & bird & cat & deer & dog & frog & horse & ship  &truck \\  
\hline\hline
\textbf{Natural Training} & 100  & 100  & 100  & 100  & 100  & 100  & 100  & 100  & 100  & 100 \\
\textbf{PGD Adv. Training} & 41.4 & 29.3 & 65.8 & 77.2 & 75.5 & 61.6 & 60.9 & 40.7 & 40.8 & 39.4\\ 
\textbf{Diff. (Adv. - Nat.)} & -58.6 & -70.7 & -34.2 & -22.8 & -24.5 & -38.4 & -39.1 & -59.3 & -59.2 & -60.6\\ 
\hline
\end{tabular}}
\label{tab:diff_wrn28}
\end{table*}

\begin{table*}[h!]
\centering
\caption{The Changes of Standard \& Robust Error in Natural \& Adversarial Training in SVHN in PreAct ResNet18.}
\scalebox{0.9}{
\begin{tabular}{c |c c c c c c c c c c} 
\hline
Std. Error & ``0'' & ``1'' & ``2'' & ``3'' & ``4'' & ``5'' & ``6'' & ``7'' & ``8'' & ``9'' \\  
\hline\hline
\textbf{Natural Training} & 3.6 & 2.7 & 3.2 & 5.8 & 4.0 & 5.1 & 3.6 & 4.6 & 6.1 & 5.3\\
\textbf{PGD Adv. Training} &8.8 &6.2  &7.9 &15.8 & 6.9 & 13.2
& 13.3 & 13.4 &19.8 &11.4\\
\textbf{Diff. (Adv. - Nat.)} & 5.2 & 3.5 & 4.8 & 10.1 & 4.9 & 8.1 & 9.6 & 8.9 & 13.6 & 6.4\\ 
\hline
\end{tabular}}

\vspace{0.2cm}
\scalebox{0.9}{
\begin{tabular}{c |c c c c c c c c c c} 
\hline
Std. Error & ``0'' & ``1'' & ``2'' & ``3'' & ``4'' & ``5'' & ``6'' & ``7'' & ``8'' & ``9'' \\   
\hline\hline
\textbf{Natural Training} & 100  & 100  & 100  & 100  & 100  & 100  & 100  & 100  & 100  & 100 \\
\textbf{PGD Adv. Training} & 47.2 & 38.8 & 49.0 & 63.4 & 41.9 & 57.1 & 62.6 & 55.2 & 73.7 & 57.1\\ 
\textbf{Diff. (Adv. - Nat.)} & -52.8 & -61.2 & -51.1 & -36.6 & -58.1 & -42.9 & -37.4 & -44.8 & -26.3 & -42.9\\ 
\hline
\end{tabular}}
\label{tab:diff_svhn}
\end{table*}

\subsection{Theoretical Proof of Section~\ref{sec:toy}}~\label{app:thy}
In this section, we formally calculate the classwise standard \& robust errors in an optimal linear classifier and an optimal linear robust classifier. Then, we present our main conclusion that robust optimization will unequally influence the performance of the two classes and therefore result in a severe performance disparity.  

\subsubsection{Proof of Theorem 1}
In this subsection, we study an optimal linear classifier which minimizes the average standard error. By calculating its standard errors, we can get the conclusion that the class ``+1'' in distribution $\mathcal{D}$ is indeed harder than class ``-1''. We first start from a lemma to calculate the weight vector of an optimal linear model. 

\begin{lemma}[Weight Vector of an Optimal Classifier]
\label{thm:optimal_nat}
For the data following the distribution $\mathcal{D}$ defined in Eq.~\ref{eq:data_dist1}, an optimal linear classifier $f_\text{nat}$ which minimizes the average standard classification error:
\begin{align*}
\begin{split}
    f_\text{nat}(x) &= \text{sign} ( \langle w_\text{nat} , x \rangle + b_\text{nat} ) \\
    \text{where}~~  w_\text{nat}, b_\text{nat} &= \argmin_{w,b} ~\text{Pr.}(\text{sign}(\langle w , x \rangle+b) \neq y)
\end{split}
\end{align*}
has the optimal weight that satisfy: $w_\text{nat} = \bf 1$. 
\end{lemma}

\begin{proof}[Proof of Lemma~\ref{thm:optimal_nat}] In the proof we will use $w = w_\text{nat}$ and $b = b_\text{nat}$ for simplicity. Next, we will prove $w_1=w_2=\dots=w_d$ by contradiction. We define $G = \{1,2,\dots, d\}$ and make the following assumption: for the optimal $w$ and $b$, we assume if there exist $w_i < w_j$ for $i\neq j$ and $i,j \in G$. Then we obtain the following standard errors for two classes of this classifier with weight $w$:
\begin{align}
\label{eq:weigh1}
\begin{split}
\mathcal{R}_\text{nat}(f;-1) &= \text{Pr}\{\sum_{k\neq i, k\neq j} w_k \gN(-\eta, \sigma_{-1}^2) +b\\& + w_i \gN(-\eta, \sigma_{-1}^2) + w_j \gN(-\eta, \sigma_{-1}^2) > 0 \} \\
\mathcal{R}_\text{nat}(f;+1) &= \text{Pr}\{\sum_{k\neq i, k\neq j} w_k \gN(+\eta, \sigma_{+1}^2) +b \\&+ w_i \gN(+\eta, \sigma_{+1}^2) + w_j \gN(+\eta, \sigma_{+1}^2) < 0 \}
\end{split}
\end{align}
However, if we define a new classier $\tilde{f}$ whose weight vector $\tilde{w}$ uses $w_j$ to replace $w_i$, we obtain the errors for the new classifier:
\begin{align}
\label{eq:weigh2}
\begin{split}
\mathcal{R}_\text{nat}(\tilde{f};-1) &= \text{Pr}\{\sum_{k\neq i, k\neq j} w_k \gN(-\eta, \sigma_{-1}^2) +b \\
&+ w_j \gN(-\eta, \sigma_{-1}^2) + w_j \gN(-\eta, \sigma_{-1}^2) > 0 \} \\
\mathcal{R}_\text{nat}(\tilde{f};+1) &= \text{Pr}\{\sum_{k\neq i, k\neq j} w_k \gN(+\eta, \sigma_{+1}^2) +b \\
&+ w_j \gN(+\eta, \sigma_{+1}^2) + w_j \gN(+\eta, \sigma_{+1}^2) < 0 \}.    
\end{split}
\end{align}
By comparing the errors in Eq~\ref{eq:weigh1} and Eq~\ref{eq:weigh2}, it can imply the classifier $\tilde{f}$ has smaller error in each class. Therefore, it contradicts with the assumption that $f$ is the optimal classifier with least error. Thus, we conclude for an optimal linear classifier in natural training, it must satisfies $w_1=w_2=\dots=w_d$ and $w = \bf 1$. 
\end{proof}
Given the results in Lemma 1, we can calculate the errors of classifiers by only calculating the interception term $b_\text{nat}$ and  $b_\text{rob}$. Recall Theorem~\ref{thm:error_natural_train}, we explicitly calculate the classwise errors of an optimal classifier which minimizes the average standard error. 

\begingroup
\def\thetheorem{\ref{thm:error_natural_train}}
\begin{theorem} 
For a data distribution $\mathcal{D}$ in Eq.~\ref{eq:data_dist1}, for the optimal linear classifier $f_\text{nat}$ which minimizes the average standard classification error, it has the intra-class standard error for the two classes:
\begin{align*}
\label{eq:nat_error}
\small
\begin{split}
&\mathcal{R}_\text{nat}(f_\text{nat},-1) 
= \text{Pr.} \{\mathcal{N}(0,1)\leq  A  - K\cdot\sqrt{A^2 + q(K)}  \}
\\
& \mathcal{R}_\text{nat}(f_\text{nat},+1) 
=  \text{Pr.} \{\mathcal{N}(0,1)\leq  -K\cdot A + \sqrt{A^2 + q(K)}\}
\end{split}
\end{align*}
where $A = \frac{2}{K^2-1}\frac{\sqrt{d}\eta}{\sigma}$ and $q(K) = \frac{2\log K}{K^2-1}$ which is a positive constant and only depends on $K$,  As a result, the class ``+1'' has a larger standard error:
\begin{equation*}
    \mathcal{R}_\text{nat}(f_\text{nat},-1) <\mathcal{R}_\text{nat}(f_\text{nat},+1).
\end{equation*}
\end{theorem}

\begin{proof}[Proof of Theorem~\ref{thm:error_natural_train}]
From the results in Lemma 1, we define our optimal linear classifier to be $f_\text{nat}(x) = \text{sign(}\sum_{i=1}^{d}x_i+b_\text{nat})$. Now, we calculate the optimal $b_\text{nat}$ which can minimize the average standard error:
\begin{align}\label{eq:seperate_std_error}
\begin{split}
    R_\text{nat}(f) &= \text{Pr.}\{f(x) \neq y\}  \\
    & \propto   \text{Pr}\{f(x)=1 | y=-1\} + \text{Pr}\{f(x)=-1 | y=1\} \\
    &  =  \text{Pr.}\{\sum_{i=1}^{d}x_i+b_\text{nat}  > 0 | y=-1\}\\
    &+  \text{Pr}\{\sum_{i=1}^{d}x_i+b_\text{nat} < 0 | y=+1\} \\
    & = \text{Pr.} \{\gN(0,1)< - \frac{\sqrt{d}\eta}{\sigma}\ + \frac{1}{\sqrt{d}\sigma} \cdot b_\text{nat} \} \\
    & + \text{Pr.} \{\gN(0,1)<  -\frac{\sqrt{d}\eta}{K\sigma}\ - \frac{1}{K\sqrt{d}\sigma} \cdot b_\text{nat} \} \\
\end{split}
\end{align}
The optimal $b_\text{nat}$ to minimize $\mathcal{R}_\text{nat}(f)$ is achieved at the point that $\frac{\partial \mathcal{R}_\text{nat}(f)}{\partial b_\text{nat} } = 0$. Thus, we find the optimal $b_\text{nat}$:
\begin{align}\label{eq:b_nat}
b_\text{nat} = \frac{K^2+1}{K^2-1}\cdot d\eta - K\sqrt{\frac{4d^2\eta^2}{(K^2-1)^2} + q(K)d\sigma^2}
\end{align} and $q(K) = \frac{2\log K}{K^2-1}$ which is a positive constant and only depends on $K$. By incorporating the optimal $b_\text{nat}$ into Eq.~\ref{eq:seperate_std_error}, we can get the classwise standard errors for the two classes:
\begin{align*}
\begin{split}
&\mathcal{R}_\text{nat}(f_\text{nat},-1) 
= \text{Pr.} \{\mathcal{N}(0,1)\leq  A  - K\cdot\sqrt{A^2 + q(K)}  \}
\\
& \mathcal{R}_\text{nat}(f_\text{nat},+1) 
=  \text{Pr.} \{\mathcal{N}(0,1)\leq  -K\cdot A + \sqrt{A^2 + q(K)}\}
\end{split}
\end{align*}
where $A = \frac{2}{K^2-1}\frac{\sqrt{d}\eta}{\sigma}$. Since $q(K)>0$, we have the direct conclusion that $\mathcal{R}_\text{nat}(f;-1)<\mathcal{R}_\text{nat}(f;+1)$.
\end{proof}

\subsubsection{Proof of Theorem 2}
In this subsection, we focus on calculating the errors of robust classifiers which minimize the average robust errors of the model. By comparing natural classifiers and robust classifiers, we get the conclusion that robust classifiers can further exacerbate the model's performance on the ``harder'' class. Similar to Section A.2.1, we start from a Lemma to show an optimal robust classifier $f_\text{rob}$ has a weight vector $w_\text{rob} = \bf 1$. 
\begin{lemma}[Weight Vector of an Optimal Robust Classifier]
\label{thm:optimal_nat}
For the data following the distribution $\mathcal{D}$ defined in Eq.~\ref{eq:data_dist1}, an optimal linear classifier $f_\text{nat}$ which minimizes the average standard classification error:
\begin{align*}
\begin{split}
     f_{\text{rob}}(x) &= \text{sign} ( \langle w_{\text{rob}} , x \rangle + b_{\text{rob}} ) \\
    \text{where}~~  w_{\text{rob}}, b_{\text{rob}} &= \argmin_{w,b} ~\text{Pr.}(\exists    \delta, ||\delta||\leq\epsilon,\\ &~~\text{s.t.}~ \text{sign}(\langle w,x+\delta\rangle + b) \neq y).
\end{split}
\end{align*} 
has the optimal weight which satisfy: $w_\text{rob}=\bf 1$.
\end{lemma}
We leave the detailed proof out in the paper because it can be proved in the similar way as the proof of Lemma 1. Recall Theorem~\ref{thm:error_adv_train}, we formally calculate the standard errors of an optimal robust linear classifier.
\begingroup
\def\thetheorem{\ref{thm:error_adv_train}}
\begin{theorem} 
For a data distribution $\mathcal{D}$ in Eq.~\ref{eq:data_dist1}, the optimal robust linear classifier $f_\text{rob}$ which minimizes the average robust error with perturbation margin $\epsilon < \eta$, it has the intra-class standard error for the two classes:
\begin{align}
\label{eq:nat_error}
\small
\begin{split}
&\mathcal{R}_\text{nat}(f_\text{rob},-1)\\ 
= &\text{Pr.} \{\mathcal{N}(0,1)\leq  B  - K\cdot\sqrt{B^2 + q(K)} -\frac{\sqrt{d}}{\sigma}\epsilon  \}
\\
& \mathcal{R}_\text{nat}(f_\text{rob},+1)\\ 
=  &\text{Pr.} \{\mathcal{N}(0,1)\leq  -K\cdot B + \sqrt{B^2 + q(K)}-\frac{\sqrt{d}}{K\sigma}\epsilon\}
\end{split}
\end{align}
where $B = \frac{2}{K^2-1}\frac{\sqrt{d}(\eta-\epsilon)}{\sigma}$ and  $q(K) = \frac{2\log K}{K^2-1}$ is a positive constant and only depends on $K$, 
\end{theorem}

\label{appendix:WRN28}
\begin{table*}[t!]
\centering
\caption{Average \& worst-class standard error, boundary error and robust error for various algorithms on CIFAR10 under WRN28.}
\resizebox{0.85\textwidth}{!}
{
\begin{tabular}{c |c c |c c |c c} 
\hline
& \textbf{Avg. Std.} & \textbf{Worst Std.} & \textbf{Avg. Bndy.} & \textbf{Worst Bndy.} & \textbf{Avg. Rob.} & \textbf{Worst Rob.}  \\  
\hline\hline
\textbf{PGD Adv. Training}& 14.0 & 29.3 & 38.1 & 53.0 & 52.2 & 78.8\\ 
\textbf{TRADES$(1/\lambda = 1)$} &\textbf{12.6} & 25.2 & 40.2 & 58.7 & 52.8 & 76.7\\ 
\textbf{TRADES$(1/\lambda = 6)$} &15.5 & 29.1 & \textbf{31.8} & \textbf{45.7} & \textbf{47.3} & 71.8\\ 
\textbf{Baseline Reweight} &14.2 &26.3 &38.6 &53.7&52.8 & 77.9\\ 
\hline
\textbf{FRL(Reweight, 0.05)} &14.5 & \textbf{23.2} & 40.0 & 53.3 & 54.4 & 76.8\\ 
\textbf{FRL(Remargin, 0.05)}&15.4 & 24.9 & 38.1 & 49.6 & 53.5 & 70.5\\ 
\textbf{FRL(Reweight+Remargin, 0.05)} &15.4 & 25.0 & 37.8 & 46.7 & 53.2 & \textbf{67.1}\\ 
\hline
\textbf{FRL(Reweight, 0.07)} &14.1 & 23.8 & 39.5 & 54.1 & 53.6 & 77.0\\ 
\textbf{FRL(Remargin, 0.07)}&14.8 & 24.3 & 39.5 & 50.6 & 54.3 & 73.0\\ 
\textbf{FRL(Reweight+Remargin, 0.07)} &14.9 & 24.7 & 37.8 & 48.3 & 52.7 & 68.2\\ 
\hline
\end{tabular}}
\label{tab:debias3}
\end{table*}

\begin{table*}[h]
\centering
\caption{Average \& worst-class standard error, boundary error and robust error for various algorithms on SVHN under WRN28.}
\resizebox{0.85\textwidth}{!}
{
\begin{tabular}{c |c c |c c |c c} 
\hline
& \textbf{Avg. Std.} & \textbf{Worst Std.} & \textbf{Avg. Bndy.} & \textbf{Worst Bndy.} & \textbf{Avg. Rob.} & \textbf{Worst Rob.}  \\  
\hline\hline
\textbf{PGD Adv. Training}& 8.1 & 16.8 & 38.5 & 57.3 & 46.7 & 71.2\\ 
\textbf{TRADES$(1/\lambda = 1)$} & 8.0 & 19.6 & 40.1 & 60.0 & 48.1 & 73.3\\ 
\textbf{TRADES$(1/\lambda = 6)$} &10.6 & 23.1 & \textbf{32.1} & 52.5 & \textbf{42.7} & 70.6\\ 
\textbf{Baseline Reweight} &8.5& 16.2 &40.3&57.8&48.8&71.1\\ 
\hline
\textbf{FRL(Reweight, 0.05)} &\textbf{7.8} & 13.4 & 38.9 & 56.9 & 46.7 & 70.7\\ 
\textbf{FRL(Remargin, 0.05)}&8.4 & 13.4 & 40.8 & 52.1 & 49.2 & 65.5\\ 
\textbf{FRL(Reweight+Remargin, 0.05)} & 8.4 & \textbf{13.2} & 38.4 & 52.1 & 46.8 & \textbf{63.1} \\ 
\hline
\textbf{FRL(Reweight, 0.07)} &8.2 & 13.5 & 41.2 & 56.3 & 49.4 & 69.8\\ 
\textbf{FRL(Remargin, 0.07)}&8.6 & 14.9 & 38.8 & 51.2 & 47.4 & 67.0\\ 
\textbf{FRL(Reweight+Remargin, 0.07)} &8.2 & 13.9 & 39.9 & \textbf{50.2} & 48.1 & 65.4\\ 
\hline
\end{tabular}}
\label{tab:debias4}
\end{table*}

\begin{proof}[Proof of Theorem~\ref{thm:error_adv_train}]
From the results in Lemma~2, we define our optimal linear robust classifier to be $f_\text{rob}(x) = \text{sign(}\sum_{i=1}^{d}x_i+b_\text{rob})$. Now, we calculate the optimal $b_\text{rob}$ which can minimize the average robust error:
\begin{align}\label{eq:seperate_std_error}
\begin{split}
\mathcal{R}_\text{rob}(f)=&\text{Pr.}(\exists ||\delta||\leq \epsilon~~ \text{s.t.}~~ f(x+\delta)\neq y)\\
=& \max_{||\delta||\leq\epsilon}{\text{Pr.}(f(x+\delta)\neq y)}\\
=& \frac{1}{2}\text{Pr.}(f(x+\epsilon)\neq -1|y = -1)\\
& +\frac{1}{2} \text{Pr.}(f(x -\epsilon)\neq+1|y = +1)\\
= & \text{Pr.}\{\sum_{i=1}^{d}(x_i+\epsilon)+b_\text{rob} > 0 | y=-1\}\\
&+  \text{Pr}\{\sum_{i=1}^{d}(x_i-\epsilon)+b_\text{rob} < 0 | y=+1\} \\
= &\text{Pr.} \{\gN(0,1)< - \frac{\sqrt{d}(\eta - \epsilon)}{\sigma}\ + \frac{1}{\sqrt{d}\sigma} \cdot b_\text{rob}\} \\
&+ \text{Pr.} \{\gN(0,1)<  -\frac{\sqrt{d}(\eta-\epsilon)}{K\sigma}\ - \frac{1}{K\sqrt{d}\sigma} \cdot b_\text{rob}\} \\
\end{split}
\end{align}
The optimal $b_\text{rob}$ to minimize $\mathcal{R}_\text{rob}(f)$ is achieved at the point that $\frac{\partial \mathcal{R}_\text{rob}(f)}{\partial b_\text{rob} } = 0$. Thus, we find the optimal $b_\text{rob}$:
\begin{align}\label{eq:b_rob}
b_\text{rob} = \frac{K^2+1}{K^2-1}\cdot d(\eta-\epsilon) - K\sqrt{\frac{4d^2(\eta-\epsilon)^2}{(K^2-1)^2} + q(K)d\sigma^2}
\end{align} and $q(K) = \frac{2\log K}{K^2-1}$ which is a positive constant and only depends on $K$. By incorporating the optimal $b_\text{nat}$ into Eq.~\ref{eq:seperate_std_error}, we can get the classwise robust errors for the two classes:
\begin{align*}
\begin{split}
&\mathcal{R}_\text{rob}(f_\text{rob},-1) 
= \text{Pr.} \{\mathcal{N}(0,1)\leq  B  - K\cdot\sqrt{B^2 + q(K)}  \}
\\
& \mathcal{R}_\text{rob}(f_\text{rob},+1) 
=  \text{Pr.} \{\mathcal{N}(0,1)\leq  -K\cdot B + \sqrt{B^2 + q(K)}\}
\end{split}
\end{align*}
where $B = \frac{2}{K^2-1}\frac{\sqrt{d}(\eta-\epsilon)}{\sigma}$. As a direct result, the classwise standard errors for the two classes:
\begin{align*}
\begin{split}
&\mathcal{R}_\text{nat}(f_\text{rob},-1)\\ 
= &\text{Pr.} \{\mathcal{N}(0,1)\leq  B  - K\cdot\sqrt{B^2 + q(K)} -\frac{\sqrt{d}}{\sigma}\epsilon  \}
\\
& \mathcal{R}_\text{nat}(f_\text{rob},+1)\\ 
=  &\text{Pr.} \{\mathcal{N}(0,1)\leq  -K\cdot B + \sqrt{B^2 + q(K)}-\frac{\sqrt{d}}{K\sigma}\epsilon\}.
\end{split}
\end{align*}
\end{proof}

\subsubsection{Proof of Corollary 1}
Giving the results in Theorem 1 and Theorem 2, we will show that a robust classifier will exacerbate the performance of the class ``+1'' which originally has higher error in a naturally trained model. In this way, we can get the conclusion that robust classifiers can cause strong disparity, because it exacerbates the ``difficulty'' difference among classes.
\begingroup
\def\thecorollary{\ref{thm:enlarge_disparity}}
\begin{corollary}
Adversarially Trained Models on $\mathcal{D}$ will increase the standard error for class ``+1'' and reduce the standard error for class ``-1'':
\begin{align*}
\begin{split}
\mathcal{R}_\text{nat}(f_\text{rob},-1) <\mathcal{R}_\text{nat}(f_\text{nat},-1).  \\
\mathcal{R}_\text{nat}(f_\text{rob},+1) >\mathcal{R}_\text{nat}(f_\text{nat},+1).  
\end{split}
\end{align*}
\end{corollary}
\begin{proof}[Proof of Corollary~\ref{thm:enlarge_disparity}]
From the intermediate results in Eq.\ref{eq:b_nat} and Eq.~\ref{eq:b_rob} in the proofs of Theorem 1 and Theorem 2, we find the only difference between a naturally trained model $f_\text{nat}$ and a robust model $f_\text{rob}$ is about the interception term $b_\text{nat}$ and $b_\text{rob}$. Specifically, we denote $g(\cdot)$ is the function of the interception term, then we have the results: \begin{align*}
\begin{split}
b_\text{nat} &= \frac{K^2+1}{K^2-1}\cdot d\eta - K\sqrt{\frac{4d^2\eta^2}{(K^2-1)^2} + q(K)d\sigma^2}:=g(\eta) \\
b_\text{rob} &= \frac{K^2+1}{K^2-1}\cdot d(\eta-\epsilon) - K\sqrt{\frac{4d^2(\eta-\epsilon)^2}{(K^2-1)^2} + q(K)d\sigma^2}\\
&:=g(\eta-\epsilon) \\
\end{split}
\end{align*}
Next, we show that the function $g(\cdot)$ is a monotone increasing function between 0 and $\eta$:
$$
\frac{d g(\eta)}{d\eta} \geq \frac{K^2+1}{K^2-1}d-K\frac{\frac{4}{(K^2-1)^2}d^2\cdot 2\eta}{2\sqrt{\frac{4}{(K^2-1)^2}d^2\eta^2}} = \frac{K-1}{K+1}d>0
$$
As a direct results, we have the interception terms: $b_\text{nat} > b_\text{rob}$. This will result a linear classifier $f(x) = \text{sign}(\langle\textbf{1}^T,x\rangle+b)$ present more samples in the overall distribution to be class ``-1''. As a result, we have the conclusion:
\begin{align*}
\begin{split}
\mathcal{R}_\text{nat}(f_\text{rob},-1) <\mathcal{R}_\text{nat}(f_\text{nat},-1).  \\
\mathcal{R}_\text{nat}(f_\text{rob},+1) >\mathcal{R}_\text{nat}(f_\text{nat},+1).  
\end{split}
\end{align*}
\end{proof}

\subsection{Robust / Non-Robust Features}
In Section~\ref{sec:binary_example}, we discussed a theoretical example where adversarial training will unequally treat the two classes in the distribution $\mathcal{D}$, which will increase the standard error of one class and decrease the error for the other one. 
However, in the real applications of deep learning models, we always observe that each class's error will increase after adversarial training. In this subsection, motivated by the work~\cite{tsipras2018robustness, ilyas2019adversarial}, we extend the definition of $\mathcal{D}$, to split the features into two categories: robust features and non-robust features, where adversarial training will increase the standard errors for both classes. Formally, the data distribution $\mathcal{D}'$ is defined as as following:
\begin{spreadlines}{0.8em}
\begin{align}
\small
\label{eq:data_dist2}
\begin{split}
    & y \stackrel{u.a.r}{\sim} \{-1, +1\},~~~ \theta = ( \overbrace{\eta,...,\eta}^\text{dim = d}, \overbrace{\gamma,...,\gamma}^\text{dim = m}),~~~
    \\ &x \sim  
    \begin{cases}
      ~~\mathcal{N}~(\theta, \sigma_{+1}^2I) ~~~~&\text{if $y= + 1$}\\
      ~~\mathcal{N}~(-\theta, \sigma_{-1}^2I) ~~~~&\text{if $y= - 1$}\\
    \end{cases} 
\end{split}
\end{align}
\end{spreadlines}
where in the center vector $\theta$, it includes robust features with scale $\eta > \epsilon$, and non-robust features with scale $\gamma < \epsilon$. Here we specify that non-robust feature space has much higher dimension than robust feature space ($m>>d$) and there is a $K$-factor difference between the variance of two classes: $\sigma_{+1} = K\cdot\sigma_{-1}$. From the main results in the work~\cite{tsipras2018robustness}, it is easy to get that each class's standard error will increase after adversarial training. In the following theory, we will show that in distribution $\mathcal{D}'$ , adversarial training will increase the error for the class ``+1'' by a larger rate than the class ``-1''.

\textbf{Theorem 3.}
\textit{For a data distribution $\mathcal{D'}$ in Eq.~\ref{eq:data_dist2}, the robust optimizer $f_\text{rob}$ increases the standard error of class ``+1'' by a larger rate than the increase of the standard error of class ``-1':
\begin{align}
\label{eq:thm3}
\begin{split}
    \mathcal{R}_\text{nat}(f_\text{rob};+1) - \mathcal{R}_\text{nat}(f_\text{nat};+1) >\\
\mathcal{R}_\text{nat}(f_\text{rob};-1) - \mathcal{R}_\text{nat}(f_\text{nat};-1)
\end{split}
\end{align}
}
\begin{proof}[Proof of Theorem 3]
The proof of Theorem 3 resembles the process of the proofs of Theorem 1 and Theorem 2, where we first calculate the classwise standard errors for each model. Note that from the work~\cite{tsipras2018robustness}, an important conclusion is that a natural model $f_\text{nat}$ uses both robust and non-robust features for prediction. While, a robust model $f_\text{rob}$ only uses robust features for prediction (a detailed proof can be found in Section 2.1 in ~\cite{tsipras2018robustness}). Therefore, we can calculate the classwise standard errors for both classes of a natural model $f_\text{nat}$:
\begin{align*}
\begin{split}
&\mathcal{R}_\text{nat}(f_\text{nat},-1) 
= \text{Pr.} \{\mathcal{N}(0,1)\leq  A  - K\cdot\sqrt{A^2 + q(K)}  \}
\\
& \mathcal{R}_\text{nat}(f_\text{nat},+1) 
=  \text{Pr.} \{\mathcal{N}(0,1)\leq  -K\cdot A + \sqrt{A^2 + q(K)}\}
\end{split}
\end{align*}
where $A =\frac{2}{\sigma(K^2-1)} \sqrt{m\gamma^2 + d\eta^2}$. The classwise standard errors of a robust model $f_\text{rob}$ are:
\begin{align*}
\small
\begin{split}
&\mathcal{R}_\text{nat}(f_\text{rob},-1)\\ 
= &\text{Pr.} \{\mathcal{N}(0,1)\leq  B  - K\cdot\sqrt{B^2 + q(K)} -\frac{\sqrt{d}}{\sigma}\epsilon  \}
\\
& \mathcal{R}_\text{nat}(f_\text{rob},+1)\\ 
=  &\text{Pr.} \{\mathcal{N}(0,1)\leq  -K\cdot B + \sqrt{B^2 + q(K)}-\frac{\sqrt{d}}{K\sigma}\epsilon\}
\end{split}
\end{align*}
where $B =\frac{2}{\sigma(K^2-1)} \sqrt{d(\eta-\epsilon)^2}$. Next, we compare the standard error increase after adversarial training between the two classes. We have the results:
\begin{align*}
\begin{split}
&(\mathcal{R}_\text{nat}(f_\text{rob};+1) - \mathcal{R}_\text{nat}(f_\text{nat};+1))  - 
(\mathcal{R}_\text{nat}(f_\text{rob};-1) - \mathcal{R}_\text{nat}(f_\text{nat};-1))\\
&> (K+1) ((A-B) + (\sqrt{B^2 +q(K)} - \sqrt{A^2 +q(K)}))\\
&\propto(\sqrt{B^2 + q(K)} - B) - (\sqrt{A^2 + q(K)} - A)
\end{split}
\end{align*}
because $A$ includes high dimensional non-robust feature and $A>>B$, the equation above is positive and we get the conclusion as in Eq.~\ref{eq:thm3}.
\end{proof}

\subsection{Fairness Performance on WRN28}
Table~\ref{tab:debias3} and Table~\ref{tab:debias4} presents the empirical results to validate the effectiveness of FRL algorithms under the WRN 28 model. The implementation details resemble those in Section~\ref{sec:setup}. In the training, we start FRL from a pre-trained robust model (such as PGD-adversarial training), and run FRL with model parameter learning rate 1e-3 and hyperparameter learning rate $\alpha_1 = \alpha_2 = 0.05$ in the first 40 epochs. Then we decay the model parameter learning rate and the hyperparameter learning rate by $0.1$ every 40 epochs. From the results, we have the similar observations as these for PreAct ResNet18 models, which is that FRL can help to improve the worst-class standard performance and robustness performance, such that the unfairness issue is mitigated. In particular, FRL~(Reweight) is usually the most effective way to equalize the standard performance, but not sufficient to equalize the boundary errors and robust errors. FRL~(Reweight + Remargin) is usually the most effective way to improve robustness for the worst class.

\end{document}